\documentclass[twoside]{article}
\usepackage[accepted]{aistats2026}

\usepackage{amsmath}
\usepackage{amsthm}
\usepackage{amssymb}
\usepackage{graphicx}
\usepackage{xcolor}
\usepackage{mdframed}
\usepackage[colorlinks=true, allcolors=blue]{hyperref}
\usepackage[capitalize]{cleveref}

\usepackage[numbers]{natbib}

\usepackage{tikz,pgfplots}
\pgfplotsset{compat=1.17}
\usepackage{sansmath}

\title{Proof of HP Transfer for a Linear Network}

\newmdtheoremenv[
  backgroundcolor=blue!3,
  linecolor=black!60,
  linewidth=0.8pt,
  roundcorner=4pt,
  innertopmargin=6pt,
  innerbottommargin=6pt,
  innerleftmargin=8pt,
  innerrightmargin=8pt
]{thm}{Theorem}

\newtheorem{lemma}{Lemma}
\newtheorem{definition}{Definition}

\crefname{thm}{Theorem}{theorems}

\newcommand{\reals}{\mathbb{R}}
\newcommand{\normal}{\mathcal{N}}
\newcommand{\E}{\mathbb{E}}
\newcommand{\loss}{\mathcal{L}}
\newcommand{\bigO}{\mathcal{O}}
\newcommand{\data}{\mathcal{D}}
\newcommand{\Prob}{\mathbb{P}}
\newcommand{\tr}{\operatorname{Tr}}

\begin{document}

\twocolumn[

\aistatstitle{A Proof of Learning Rate Transfer under $\mu$P}

\aistatsauthor{ Soufiane Hayou }

\aistatsaddress{ Department of Applied Mathematics and Statistics \\
Johns Hopkins University} ]

\begin{abstract}
    We provide the first proof of learning rate transfer with width in a linear multi-layer perceptron (MLP) parametrized with $\mu$P, a neural network parameterization designed to ``maximize'' feature learning in the infinite-width limit. We show that under $\mu P$, the optimal learning rate converges to a \emph{non-zero constant} as width goes to infinity, providing a theoretical explanation to learning rate transfer. In contrast, we show that this property fails to hold under alternative parametrizations such as Standard Parameterization (SP) and Neural Tangent Parametrization (NTP). We provide intuitive proofs and support the theoretical findings with extensive empirical results.
\end{abstract}

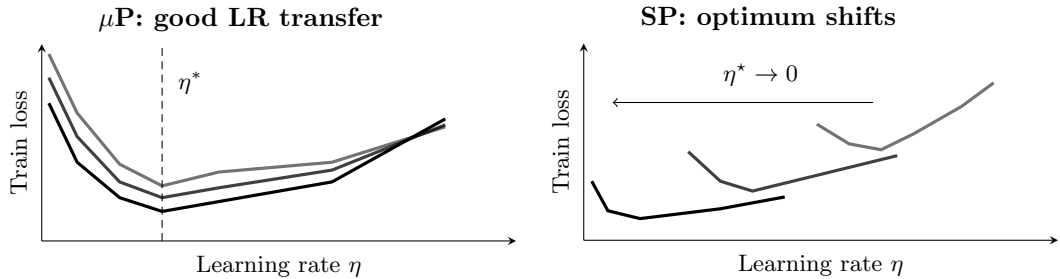
\begin{figure*}[htp]
\centering
\begin{tikzpicture}[every node/.style={font=\small}]
  \begin{axis}[
    name=muP, width=0.46\linewidth, height=4.2cm,
    xmin=0.015, xmax=0.35, ymin=0.0, ymax=1.0,
    axis x line=bottom, axis y line=left,
    xtick=\empty, ytick=\empty,
    xlabel={Learning rate $\eta$}, ylabel={Train loss},
    clip=false
  ]
    \addplot[no markers, very thick, opacity=1.00] coordinates {
      (0.02,0.7) (0.04,0.40) (0.07,0.22) (0.10,0.15) (0.14,0.20) (0.22,0.3) (0.30,0.62)
    };
    \addplot[no markers, very thick, opacity=0.75] coordinates {
      (0.02,0.83) (0.04,0.53) (0.07,0.30) (0.10,0.22) (0.14,0.27) (0.22,0.36) (0.30,0.59)
    };
    \addplot[no markers, very thick, opacity=0.55] coordinates {
      (0.02,0.95) (0.04,0.65) (0.07,0.39) (0.10,0.28) (0.14,0.35) (0.22,0.40) (0.30,0.58)
    };

    \draw[densely dashed] (axis cs:0.10,0.0) -- (axis cs:0.10,1.0);
    \node[anchor=south] at (axis cs:0.12,+0.7) {$\eta^*$};

    
    \node[anchor=south west,font=\bfseries] at (rel axis cs:0.1,1.02) {$\mu$P: good LR transfer};
  \end{axis}

  \begin{axis}[
    at={(muP.east)}, xshift=0.9cm, yshift=-1.3cm,
    width=0.46\linewidth, height=4.2cm,
    xmin=0.005, xmax=0.30, ymin=0.0, ymax=1.0,
    axis x line=bottom, axis y line=left,
    xtick=\empty, ytick=\empty,
    xlabel={Learning rate $\eta$}, ylabel={Train loss},
    clip=false
  ]
    \addplot[no markers, very thick, opacity=1.00] coordinates {
      (0.01,0.30) (0.02,0.15) (0.04,0.11) (0.08,0.15) (0.09,0.16) (0.13,0.22)
    }; 
    \addplot[no markers, very thick, opacity=0.75] coordinates {
      (0.07,0.45) (0.09,0.30) (0.11,0.25) (0.15,0.33) (0.17,0.37) (0.20,0.43)
    }; 
    \addplot[no markers, very thick, opacity=0.55] coordinates {
      (0.15,0.59) (0.17,0.49) (0.19,0.46) (0.21,0.54) (0.24,0.68) (0.26,0.80)
    }; 

    \draw[->] (axis cs:0.185,0.7) -- (axis cs:0.023,0.7);
    \node[anchor=south] at (axis cs:0.115,0.74) {$\eta^\star \to 0$};

    \node[anchor=south west,font=\bfseries] at (rel axis cs:0.1,1.02) {SP: optimum shifts};
  \end{axis}
\end{tikzpicture}
\caption{\textbf{Conceptual illustration of learning-rate transfer.}
Left: Under $\mu$P, loss curves across widths share (approximately) the same optimal learning rate $\eta^*$.
Right: Under SP, the optimal learning rate $\eta^\star_n$ shifts toward $0$ as width grows.
Curves illustrating different widths (darker $\Rightarrow$ wider).}
\label{fig:intro-lr-transfer-concept}
\end{figure*}

\section{Introduction}
    The recent successes in AI are mostly fueled by scale: large neural networks trained on large corpuses of data. Given a fixed training dataset, the size of a neural network can be scaled by increasing the width (hidden dimension) and/or depth (number of layers). As we scale these dimensions, several hyperparameters (HPs) must be adjusted with scale to avoid numerical overflows.  Motivated by this empirical observation, several works have explored the large-width limit of neural networks and its impact on optimal HPs. \citet{he2016deep} introduced the ``$1/\textrm{fan-in}$'' initialization which normalizes the weights to achieve order one activations as width grows (Note that \citet{Neal1996} was the first to introduce the ``$1/\textrm{fan-in}$'' initialization in the context of Bayesian neural networks). The Neural Tangent Kernel (NTK, \citep{jacot2018neural}) was one of the first attempts to understand training dynamics of large-width neural networks. The authors showed that under the neural tangent parametrization, training dynamics converge to a kernel regime in the infinite-width limit, a phenomenon known as lazy training \citep{chizat2020lazytrainingdifferentiableprogramming}. In this regime, neural features are almost identical to their values at initialization and training dynamics can be linearized around initialization. It quickly became clear that NTK regime does not represent practical training of neural network, which exhibit significant feature learning. \citet{yang2021tensor} reverse-engineered this problem by investigating neural parametrizations that result in feature learning in the infinite-width limit and introduced the Maximal Update Parametrization ($\mu$P) which sets precise scaling exponents for the initialization and learning rate. A nice by-product of $\mu$P is HP transfer, or where optimal HPs seem to converge as width increases, a very useful property since it allows tuning HPs on relatively small models and using them for larger models with no additional tuning cost (see \cref{fig:intro-lr-transfer-concept} for a conceptual illustration). The authors conjectured that HP transfer resulted from the fact that $\mu$P achieves ``maximal'' feature learning, and therefore the limiting dynamics are ``optimal'' in the sense that no other limit (corresponding to other parametrizations) is better in terms of training loss, thus leading to the convergence of the optimal HPs as width grows. While this intuition is valid to some extent, to the best of our knowledge, no rigorous proof of HP transfer exists in the literature. 

    Perhaps the most important hyperparameter is the learning rate, which generally requires some tuning in practice. Motivated by this, we focus on \emph{learning rate transfer} in this work and present the first proof for this phenomenon in deep linear networks parametrized with $\mu$P. Specifically, we consider a linear Multi-Layer Perceptron (MLP) and show that at training steps $t$, the optimal learning rate converges to a non-zero constant as width goes to infinity, providing a theoretical proof for learning rate transfer observed in practice. Our proof is based on the observation that with linear MLPs, the loss function at any training step can be expressed as a polynomial function of the learning rate. We study convergence dynamics of these polynomials and their roots and conclude on the convergence of the optimal learning rate as width goes to infinity. We further show that other parametrizations such as Standard Parametrization (SP) (and Neural Tangent Parametrization (NTP)) lead to significant shift in optimal learning rate as width grows, thus requiring expensive tuning.  

The paper is structured as follows. In \cref{sec:setup}, we introduce notation and definitions. In \cref{sec:step_1}, we provide a full characterization of LR transfer after one step and study the convergence rate of the optimal LR. In \cref{sec:general_t}, we provide a proof for LR transfer for general step $t$. In both \cref{sec:step_1} and \cref{sec:general_t}, extensive simulations are provided to support the theoretical results. In \cref{sec:exps}, we provide additional empirical results with varying setups: activation function, optimizer, depth, training time.

\subsection{Related work}

\paragraph{Infinite-width analysis.} There is a rich literature on the theory of infinite-width neural networks. The first works on infinite-width theory are related to approximation results showing that neural networks are universal approximators when the width to infinity (see e.g. \cite{HORNIK1989359, Cybenko1989}). Perhaps the first methodological work on infinite-width neural networks was a study of priors in large-width Bayesian neural network by \citet{Neal1996}, where the author studied how Gaussian prior should be scaled as network width increases, and showed that single-layer Bayesian networks converge to a Gaussian process in the infinite-width limit, a result that was later used in \cite{Williams1996ComputingWI} to compute infinite-width posteriors, and was later generalized to multi-layer networks in \citep{lee2018deepneuralnetworksgaussian,matthews2018gaussianprocessbehaviourwide}. Subsequent research has examined the impact of initialization \citep{deepinfoprop2017, hayou19activation, li2021future, daniely2017deeperunderstandingneuralnetworks}, the activation functions \citep{hayou19activation}, learning rate \citep{yang2022tensor}, batch size \citep{zhang2025doescriticalbatchsize}, etc. Others works studied how these HPs should scale with depth (assuming large-width) \citep{hayou2021stableresnet, yang2023tensorprogramsvifeature, bordelon2023depthwisehyperparametertransferresidual}. There is also a rich literature on training dynamics of infinite-width neural networks, including the literature on the neural tangent kernel \citep{jacot2018neural, hayou2022exactconvergenceratesneural, arora2019exactcomputationinfinitelywide, chizat2020lazytrainingdifferentiableprogramming, allenzhu2019convergencetheorydeeplearning}, and the literature on mean-field neural networks \citep{sirignano2019meanfieldanalysisneural, mei2019meanfieldtheorytwolayersneural, Mignacco_2021, chizat2022infinitewidthlimitdeeplinear}.

\paragraph{Hyperparameter transfer.} \citet{yang2021tensor} introduced $\mu$P, a neural parametrization that specifies how initialization and learning rate should scale with model width $n$. The authors derived this parametrization by searching for HPs that yield feature learning in the infinite-width limit, in contrast to neural tangent parametrization  which leads to a kernel regime in the limit \citep{jacot2018neural}. In particular, the authors observed that $\mu$P leads to an interesting phenomenon: HP transfer with width, where optimal HPs tend to stabilize as width increases.  It was conjectured that feature learning properties of the infinite-width limit under $\mu$P is the main factor behind HP transfer. In \cite{yang2022tensor}, the authors showed that $\mu$P yields HP transfer in Large Language Models (LLMs) of GPT-3 scale. However, other works showed mixed results on the efficacy of $\mu$P with LLMs and Diffusion model \citep{falcon2023falconseriesopenlanguage, blake2025umupunitscaledmaximalupdate, lingle2025empiricalstudymuplearning, everett2024scalingexponentsparameterizationsoptimizers, hayou2025optimalembeddinglearningrate, kosson2025weightdecaymattermup, zheng2025scalingdiffusiontransformersefficiently}. Other works include  \cite{noci2024superconsistencyneuralnetwork} where the authors studied learning rate transfer studied from the angle of Hessian geometry and its connection to the edge of stability \cite{cohen2022gradientdescentneuralnetworks}, \citep{bordelon2025deep} where the authors studied learning rate transfer in linear networks, \citep{chizat2025featurespeedformulaflexible} where the authors considered a feature based approach where learning rate transfer is automatically achieved, and other works that extended HP transfer to cover other optimizers \citep{jordan2024muon, ahn2025diondistributedorthonormalizedupdates, pethick2025trainingdeeplearningmodels}, depth scaling \citep{yang2023tensorprogramsvifeature, bordelon2023depthwisehyperparametertransferresidual, dey2025dontlazycompletepenables}, etc.

\section{Setup and Definitions}\label{sec:setup}
We consider a linear Multi-Layer Perceptron (MLP) given by
\begin{equation}\label{eq:mlp}
f(x) = V^\top W_L W_{L-1} \dots W_1 W_0 x,
\end{equation}
where $x \in \reals^d$ is the input, $W_0 \in \reals^{n \times d}$, $W_\ell \in \reals^{n\times n}$ for $\ell \in \{1,2, \dots, L\}$, and $V \in \reals^n$, are the weights. While we consider one-dimensional output, our results can be generalized to neural networks with multi-dimensional outputs.

Model \cref{eq:mlp} is trained by minimizing the quadratic loss $\loss = \frac{1}{2m} \, \sum_{i=1}^m (f(x_i) - y_i)^2$, where $\data = \{(x_i, y_i), i=1\dots m\}$ is the training dataset. For the sake of simplicity, we only train the weight matrices $W_1, W_2, \dots, W_L$, and fix $W_0$ and $V$ to their initialization values.\footnote{Our results can be extended to the case where $W_0$ and $V$ are trainable. For $\mu$P, the learning rate for $W_0$ should be parametrized as $\eta \times n$.} For weight updates, we use gradient descent (GD)
    \begin{equation}\label{eq:gd}
        W_\ell^{(t+1)} = W_\ell^{(t)} - \eta \nabla_{W_\ell^{(t)}} \loss, 
    \end{equation}
where $t \in \{1, 2, \dots, T\}$ is the step, $\eta$ is the learning rate, and $W^{(0)}_\ell$ is randomly initialized.

When training a neural network, we should first set the hyperparameters (HPs) such as initialization and learning rate. Generally speaking, as width grows, it should be expected that optimal HPs shift with width, indicating dependence on width $n$. Therefore, it makes sense to explicitly parametrize HPs as a function of width. For instance, He initialization \citep{he2016deep} sets the initialization weights as centred gaussian random variables with ``1/fan\_in'' variance, where ``fan\_in'' refers to the dimension of the previous layer, e.g. $n$ for $\ell \in \{1,2, \dots, L\}$, and $d$ for $\ell = 0$. For the learning rate, $\mu$P scaling parametrizes the learning rate as $\eta n^{-1}$ for Adam \citep{yang2022tensor} and $\eta$ for gradient descent. We call these \emph{neural parametrizations}, a notion that we formalize in the next definition.
\begin{definition}[Neural Parametrization]
A neural parametrization for model \cref{eq:mlp} specifies the constants $(\alpha_\ell)_{0 \leq \ell \leq L}, \alpha_V$, and $\alpha_{\eta}$:
\begin{itemize}
    \item Initialization: $W_0 \sim \normal(0, d^{-\alpha_0})$, $W_\ell \sim \normal(0, n^{-\alpha_\ell})$, and $V \sim \normal(0, n^{-\alpha_V})$. 
    \item Learning rate: $\eta \times n^{-c}$.  
\end{itemize}

\end{definition}
While a neural parametrization should in-principle cover all HPs (initialization, learning rate, batch size, Adam's ($\beta_1, \beta_2$), etc), we consider only the initialization and learning rate in this work. Here are two examples of such neural parametrizations: 
\begin{itemize}
    \item Standard Parametrization (SP): $\alpha_\ell = 1$ for $\ell \in \{0,\dots, L\}$, $\alpha_V = 1$, and $c=0$. SP does not specify width exponent for the learning rate, hence the choice of $c=0$. \footnote{While some works introduce a learning rate scaling for SP (see e.g. \cite{everett2024scalingexponentsparameterizationsoptimizers}), the standard parametrization represents common practice (e.g. PyTorch defaults) which do not set default scaling rules for the learning rate.}

    \item Maximal Update Parametrization ($\mu P$): $\alpha_\ell = 1$ for $\ell \in \{0,\dots, L\}$, $\alpha_V = 2$, and $c=0$. Notice that the only difference with SP is the choice of $\alpha_V = 2$. For the learning rate, $\mu$P coincides with SP when the training algorithm is GD, however, when considering Adam \citep{kingma2017adammethodstochasticoptimization}, the learning rate exponent becomes $c=1$.
\end{itemize}

\subsection{What is Learning Rate (LR) Transfer?}
In the context of $\mu$P, LR transfer refers to the \emph{stability of optimal LR as model width grows}. Let $\eta_n$ be the optimal learning rate for neural network \cref{eq:mlp} of width $n$; LR transfer occurs if $\eta_n$ converges to a constant $\eta_\infty > 0$. As a result of this convergence, we can expect the optimal learning rate to remain stable for $n \gg 1$, i.e. increasing model beyond some base width $n_0 \gg 1$ does not significantly affect optimal LR. This is a highly desirable property as it implies that optimal LR can be tuned on model width $n_0$ and used for models of widths $n \gg n_0$, thus reducing tuning costs. However, for such property to be useful, $\eta_n$ should converge fast enough so that considering $|\eta_n-\eta_\infty|$ is small enough for practical model widths (e.g. $n=10^3$). A recent concurrent work by \citet{ghosh2025understandingmechanismsfasthyperparameter} studied the mechanisms of fast HP transfer and connects it to the geometry of the gradients.

\paragraph{Learning rate transfer as described in \citet{yang2021tensor}.} The authors showed empirically that learning rate transfer occurs under $\mu P$. They justified this observation with the intuition that $\mu P$ is associated with ``maximal'' feature learning. Specifically, $\mu$P is the only parametrization that achieves $\Delta z = \Theta(1)$ asymptotically in width $n$ for any activation $z$ in the neural network, while other parametrizations such as Standard Parametrization (SP) and Neural Tangent Parametrization (NTP) lead to suboptimal learning dynamics as model width $n$ grows (e.g. vanishing feature updates $\Delta z = \bigO(n^{-\beta})$ or exploding feature updates $\Delta z = \Omega(n^{\alpha})$ for some $\alpha, \beta >0$). While heuristic arguments were provided as to why learning rate transfer occurs under $\mu$P, to the best of our knowledge, no formal proof was provided showing the convergence of $\eta_n$ in the case of multi-layer neural networks.

\paragraph{Proving learning rate transfer is non-trivial.} From a mathematical perspective, proving learning rate transfer requires proving the convergence of the optimal learning rate $\eta_n$ to a non-zero constant as width goes to infinity. Optimal learning rate is (naturally) defined as the argmin of the training loss over a some set of possible values for the learning rate $\eta$. Since the loss is a random variable (from the random initialization), proving convergence of optimal learning rate requires proving convergence of the argmin of a stochastic process.

We provide the \emph{first proof to LR transfer} with width in linear MLPs of any depth (model \ref{eq:mlp}). We further show that with other parameterizations such as SP (or NTP), learning rate doesn't transfer. Let us first introduce some notation that will be consistently be used throughout the paper. 

\paragraph{Notation.}  Hereafter, $n$ will always denote model width. As $n$ grows, given sequences $c_n \in \reals$ and $d_n \in \reals^+$, we write $c_n = \bigO(d_n)$ when $c_n < \kappa d_n$ for $n$ large enough, for some constant $\kappa > 0$. We write $c_n = \Theta(d_n)$ if we have $\kappa_1 d_n\leq c_n \leq \kappa_2 d_n$ for some $\kappa_1, \kappa_2 >0$. For vector sequences $c_n = (c_n^i)_{1 \leq i \leq k} \in \reals^k$ (for some $k >0$), we write $c_n = \bigO(d_n)$ when $c_n^i = \bigO(d_n^i)$ for all $i \in [k]$, and same holds for other asymptotic notation. Finally, when the sequence $c_n$ is a vector of random variables, asymptotics are defined in the sense of the second moment ($L_2$ norm). For a vector $z\in \reals^n$, we will use the following norms: $\|z\| = \left(\sum_{i=1}^n z_i^2\right)^{1/2}$ (euclidean norm), and $\|z\|_1 = \sum_{i=1}^n |z_i|$ ($\ell_1$ norm). For two vectors $z,z' \in \reals^n$, $z' \otimes z$ denotes the outer product. Finally, all expectations in our analysis are taken with respect to random initialization weights.

The training dataset $\data$ is considered fixed, and the weights $(W_\ell)_{1\leq \ell \leq L}$ are updated with GD (\cref{eq:gd}). We use superscript $(t)$ for $t \in \{0, 1, \dots, T\}$ to denote the gradient step, e.g. $W_\ell^{(t)}$ is the weight matrix at the $\ell^{th}$ layer at training step $t$. Finally, since our goal is to study the asymptotics of the optimal learning rate, we abuse the notation and write $\loss^{(t)}_n(\eta)$ for the loss function of a neural network of width $n$ trained for $t$ steps with GD with learning rate $\eta$. Given width $n$ and training step $t$, an optimal LR can be defined as $\eta_{n}^{(t)} \in \textrm{argmin}_{\eta > 0} \loss_n^{(t)}(\eta)$. Note that the loss function $\loss^{(t)}_n$ depends on the random initialization weights, and therefore is a random variable itself. As a result, the optimal learning rate $\eta^{(t)}_n$ is also a random variable that is measurable with respect to the sigma-algebra generated by the initialization weights. When $\eta^{(t)}_n$ converges to some non-zero deterministic constant $\eta^{(t)}_\infty$ as width $n$ goes to infinity, we say that LR transfer occurs .

\begin{definition}[LR Transfer]
Let $t \in \{1,2, \dots, T\}$. We say that LR transfers with width $n$ if there exists a deterministic constant $\eta_\infty^{(t)}> 0$ such that the optimal learning rate $\eta_n^{(t)}$ converges in probability to a $\eta_\infty^{(t)}$ as $n$ goes to infinity.
\end{definition}

The condition $\eta_\infty^{(t)}> 0$ is crucial for LR transfer. In the case where $\eta_\infty^{(t)}=0$, all we can say is that $\eta_n^{(t)}$ converges to $0$ but setting the learning rate to $0$ results in no training. When $\eta_\infty^{(t)}> 0$, the limiting training loss is different by a $\Theta(1)$ factor in width $n$, i.e. achieving non-trivial feature updates.

Note that we consider convergence in probability for the definition of LR transfer, but it is equivalent to convergence in distribution since convergence in distribution to a constant implies convergence in probability. In the next section, we provide a comprehensive analysis of LR transfer for $t=1$ with explicit convergence rates. We later prove LR transfer for general step $t$.

\section{Learning Rate Transfer: Full Characterization at $t=1$}\label{sec:step_1}
We characterize the asymptotic behavior of the optimal learning rate after one gradient step. We show that under $\mu$P, LR transfer occurs. For other parametrizations such as SP and NTP, the optimal learning rate converges to zero or diverges, respectively, which implies that LR transfer doesn't occur in these cases. Here, we only study $\mu$P and SP, the result for NTP is straightforward.

\subsection{Learning Rate Transfer under $\mu$P}
We assume that initialization and learning rate exponents are set according to $\mu$P, namely 
\begin{itemize}
    \item Initialization: $W_0 \sim \normal(0,d^{-1})$, $W_\ell \sim \normal(0,n^{-1})$, and $V \sim \normal(0,n^{-2})$.
    \item Learning rate: constant $\eta >0$.
\end{itemize}

\paragraph{Intuitive analysis.}
Consider the simple case where the dataset consists of a single datapoint $(x,y)$. We will later state the result for general dataset size. The loss function at step $t=1$ is given by $\loss_n^{(1)}(\eta) = \frac 1 2 (f^{(1)}(x) - y)^2$, and the gradients are given by rank-1 matrices
$$
\nabla_{W_\ell} \loss_n^{(0)} = \chi \, b_{\ell + 1} \otimes a_{\ell -1}
$$
where 
\begin{equation*}
\begin{cases}
b_\ell = (W_{\ell}^{(0)})^\top (W_{\ell +1}^{(0)})^\top \dots (W_L^{(0)})^\top V,\\
a_\ell = W_{\ell}^{(0)} \dots W_{1}^{(0)} W_0 x,\\
\chi = f^{(0)}(x) - y.
\end{cases}
\end{equation*}

At $t=1$, model output for input $x$ is given by 
\begin{align*}
f^{(1)}(x) = V^\top \left[\prod_{\ell = 1}^L (W_\ell^{(0)} - \eta \, \chi \, b_{\ell+1} \otimes a_{\ell -1}) \right] W_0 x,
\end{align*}
which can be expressed as a polynomial in $\eta$. For integers $p_2\geq p_1$, define the products 
$$J_{p_2:p_1}= W_{p_2}^{(0)} W_{p_2 -1}^{(0)} \dots W_{p_1}^{(0)},$$ 
and $J_{p_2:p_1}= I_n$ for $p_2< p_1$. We can write 
$$f^{(1)}(x) = f^{(0)}(x) + \sum_{\ell = 1}^L \phi_\ell \eta^\ell,$$
where for $k \in \{1,\dots, L\}$,
$$
\phi_k = (-\chi)^k \, \sum_{\scriptstyle{1\leq \ell_1 < \dots < \ell_k \leq L}} \|b_{\ell_k+1}\|^2 \|a_{\ell_1 - 1}\|^2 \Psi(\ell_1, \dots, \ell_k),
$$
with $\Psi$ is given by 
$$
\Psi(\ell_1, \ell_2, \dots, \ell_k) =
\prod_{j=1}^k a_{\ell_{j}-1}^\top \, J_{\ell_{j} - 1: \ell_{j-1}+1}\, b_{\ell_{j-1} + 1}
$$
for $k\geq 2$, and $\Psi(\ell) = 1$ for all $\ell$ by definition when $k=1$.

Now define the optimal learning rate for width $n$, $\eta^{(1)}_n = \operatorname{argmin}_{\eta>0} \frac 1 2 (f^{(1)}(x) - y)^2$ at step $t=1$, which  we assume to be unique for convenience.  The asymptotic behavior of $\eta^{(1)}_n$ w.r.t $n$ depends mainly on the coefficients $\phi_\ell$:

\begin{itemize}
    \item \textbf{$\ell = 1$} (the coefficient of degree 1 monomial): 
$$\phi_1 = (-\chi) \sum_{\ell = 1}^L \|b_{\ell + 1}\|^2 \|a_{\ell-1}\|^2.$$
Strong Law of Large Numbers (SLLN) as $n \to \infty$ yields convergence  to $y\,L\|x\|^2 \, d^{-1}$ almost surely.

\item \textbf{$\ell \geq 2$}: we prove that $\phi_\ell$ converges to $0$ in $\mathbb L_2$ for $\ell \geq 2$. Intuitively, the convergence of $\phi_\ell$ to $0$ is a result of the fact that $f^{(0)}(x)$ converges to zero because of the Mean-field-type initialization of the projection layer $V \sim \normal(0,n^{-2})$. We now state these results below for general dataset size $m$. 

\end{itemize}

\paragraph{Results.}
Recall the training dataset consisting of $m$ samples $\data = \{(x_i,y_i), i=1,\dots, m\}$. Similar to the notation above, define 
\begin{equation*}
\begin{cases}
    a_{\ell,i} := W_\ell W_{\ell-1}\cdots W_0 x_i,\\
    b_\ell := W_\ell^\top W_{\ell+1}^\top\cdots W_L^\top V,\\
    \chi_i:=f^{(0)}(x_i)-y_i, \textrm{ for } i \in [m],
\end{cases}
\end{equation*}

with  $a_{-1,i}:=x_i$ and $b_{L+1}:=V$ by definition. The loss at step $t=1$ is given by $\loss_n^{(1)}(\eta)=\tfrac{1}{2m}\sum_{i=1}^m (f^{(1)}(x_i)-y_i)^2$ and the gradients are weighted sums of rank-1 matrices
\begin{equation}\label{eq:m-grad}
\nabla_{W_\ell}\loss_n^{(0)}
= \frac1m\sum_{i=1}^m \chi_i\, b_{\ell+1}\otimes a^{(i)}_{\ell-1}.
\end{equation}

Model output $f^{(1)}(x)$ can be expressed as a polynomial function in learning rate $\eta$. The next result characterizes the asymptotic behavior of its coefficients.

\begin{lemma}[Asymptotic coefficients]\label{lemma:m-vanish}
Fix $x \in \reals^d$. Then, there exists random scalars $(\phi_{\ell})_{1\leq \ell \leq L}$ such that $f^{(1)}(x) = f^{(0)}(x) +\sum_{\ell = 1}^L \phi_\ell \eta^\ell$, and
for $\ell \in\{2,\dots,L\}$,
\(
\|\phi_\ell\|_{L_2}=\bigO\!\left(n^{-(\ell-1)/2}\right).
\)
Moreover, we have 
$$\phi_1 \overset{a.s.}{\underset{n\to\infty}{\longrightarrow}} \frac{L}{m}\sum_{i=1}^m y_i\, \frac{\langle x, x_i \rangle}{d}.$$
\end{lemma}

The proof of \cref{lemma:m-vanish} is provided in \cref{app:proofs} and is based on the intuition developed above. The result shows that coefficients of degree $\ell \geq 2$ vanish as $n\to \infty$ with a rate of $n^{-(\ell-1)/2}$ in width. Interestingly, only the monomial of degree one does not vanish in the limit, and converges to a deterministic constant. As a result, asymptotically, the loss is quasi-quadratic in $\eta$. This allows us to fully characterize the convergence of the optimal learning rate $\eta_n^{(1)}$ at $t=1$.

For the remainder of the paper, we define the $m\times m$ normalized input Gram matrix
\(
K=\left(d^{-1}\,\langle x_i,x_j\rangle\right)_{1\leq i,j \leq m} \in\mathbb R^{m\times m},
\), and the vector containing all outputs $y = (y_1, \dots, y_m)^\top \in \reals^m$. The next result shows LR transfer at $t=1$ and characterizes the limiting optimal learning rate and the convergence rate.

\begin{figure*}[h]
    \centering
    \includegraphics[width=0.37\linewidth]{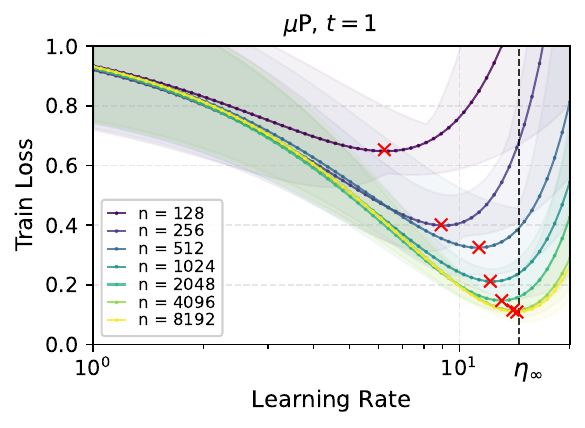}
    \includegraphics[width=0.37\linewidth]{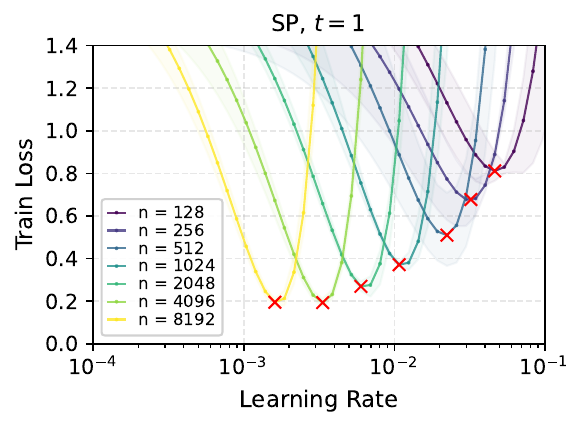}
    \includegraphics[width=0.36\linewidth]{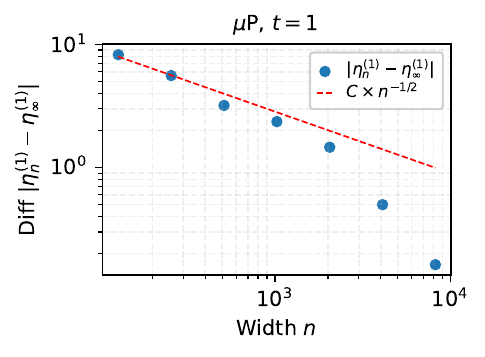}
    \includegraphics[width=0.363\linewidth]{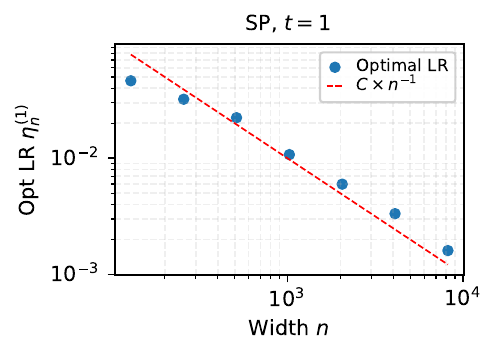}
    \caption{Optimal LR as a function of model width with 3 random seeds. \textbf{(Top)} Train loss as function of LR $\eta^{(1)}_n$ at $t=1$ for both $\mu$P and SP. \textbf{(Bottom)} Convergence of optimal LR $\eta^{(1)}_n$ as width grows.}
    \label{fig:one_step}
\end{figure*}

\begin{thm}[LR transfer at $t=1$]\label{thm:m-argmin}
Assume that $Ky\neq 0$ and define 
$$
\eta_\infty^{(1)}
\;=\; \frac{m}{L}\,
\frac{y^\top K y}{\| K y\|^2}.
$$ 
Then, for any compact interval $I \subset [0,\infty)$ containing $\eta_\infty^{(1)}$, and any $\eta^{(1)}_n \in \textrm{argmin}_{\eta \in I} \loss^{(1)}_n(\eta)$, we have
\[
\eta_n^{(1)}-\eta_\infty^{(1)} = O_{\mathbb P}(n^{-1/2}).
\]
\end{thm}

\cref{thm:m-argmin} shows convergence of the optimal LR to a deterministic limit $\eta_\infty^{(1)} > 0$, thus proving learning rate transfer at $t=1$. The convergence rate is $\bigO(n^{-1/2})$ which is expected with large-width asymptotics. The compact interval $I$ can be arbitrarily large as long as it contains $\eta_\infty^{(1)}$. The proof is provided in \cref{app:proofs} and is based on several technical lemmas used to control large-width deviations.

To verify LR transfer empirically, we trained a three layers linear MLP parametrized with $\mu$P with varying widths $n \in \{2^k, k=7,\dots, 13\}$ with GD. Training data consists of synthetically generated data $y = w^\top x + \epsilon$ where $x \sim \mathcal N(0,I_d)$ and $w \sim \mathcal N(0,d^{-1} I_d)$ ($d=1$), and $\epsilon \sim \mathcal N(0,0.01)$. We use $N=1000$ samples for training (see \cref{sec:exps} for more details about experimental setup). \cref{fig:one_step} (top left) shows optimal learning rate with $\mu$P as a function of width. Convergence analysis is displayed in the bottom left figure. We observe convergence of the optimal LR $\eta^{(1)}_n$ to the theoretical value $\eta^{(1)}_\infty$ as $n$ grows which confirms the theoretical findings. Interestingly, the empirical convergence rate seems to match the theoretical prediction of $n^{-1/2}$ up to width $n=1024$ then becomes faster for larger widths. This indicates that our upperbound $\bigO(n^{-1/2})$ is likely not tight for large widths and we currently do not have an explanation for this sudden change in convergence rate.\footnote{Note that LR transfer is most usefull when convergence is fast.}

\subsection{Failure of LR Transfer under SP/NTP}
With standard parametrization, the only difference with $\mu$P lies in how the projection layer weight $V$ is initialized: $V\sim\mathcal N(0,n^{-1})$ for SP, instead $n^{-2}$ variance with $\mu$P. Other weights are initialized as
$W_0\sim\mathcal N(0,d^{-1})$ and $W_\ell\sim\mathcal N(0,n^{-1})$ for $\ell=1,\dots,L$, and the learning rate is a constant $\eta$ that is not parametrized with width. Note that this is only true for GD (and SGD). For Adam \citep{kingma2017adammethodstochasticoptimization}, SP and $\mu$P also differ in the learning rate exponent ($c=1$ for $\mu$P and $c=0$ for SP). 

The next result shows that optimal learning rate with SP converges to $0$ as width grows, suggesting that LR transfer cannot occur under this parametrization.
\begin{thm}[No LR transfer under SP]\label{thm:std-noHP}
Let $\bar \eta>0$ be an arbitrary constant, and $\eta_n^{(1)}\in\arg\min_{\eta \in [0,\bar \eta]} \loss_n^{(1)}(\eta)$ for the one-step loss, and assume $Ky\neq 0$. Then  $\eta_n^{(1)} \xrightarrow{\mathbb P} 0$ as $n\to \infty$.
\end{thm}

Intuitively, because of the $n^{-1}$ variance in $V$ initialization, all coefficients are amplified by a factor \(\sqrt{n}\) compared to $\mu$P, so the optimal one-step LR compensates for that growth. The proof of \cref{thm:std-noHP} is provided in \cref{app:proofs}.

With NTP \citep{jacot2018neural}, the opposite occurs. To see this, recall that NTP involves multipliers in front of the weights. Specifically, we take \(\widetilde W_\ell,\widetilde V\) with i.i.d.\ $\mathcal N(0,1)$ entries and define
\[
W_0 = \frac{1}{\sqrt{d}} \widetilde W_0, \quad W_\ell=\frac{1}{\sqrt n}\,\widetilde W_\ell, \quad
V=\frac{1}{\sqrt n}\,\widetilde V .
\]
This is distributionally identical to $W_\ell\sim\mathcal N(0,n^{-1})$ and $V\sim\mathcal N(0,n^{-1})$. However, the ``effective'' learning rate is now scaled by the $n^{-1/2}$ factor in front of the weights, which leads to a kernel regime in the limit (no feature learning). Hence, optimal learning rate tends to compensate for this down-scaling by blowing-up with width.

\cref{fig:one_step} (right) shows the optimal LR as a function of width $n$ under SP. Unlike with $\mu$P, the optimal LR $\eta^{(1)}_n$ does not exhibit convergence to a non-zero constant, but rather shifts significantly with width, converging to zero. Therefore, LR transfer does not occur with SP. The bottom right figure shows the empirical convergence rate which seems to be faster than $n^{-1/2}$ and closer to $n^{-1}$.

\section{Learning Rate Transfer at any Step}\label{sec:general_t}
We generalize the results from the previous section and prove LR transfer for general gradient step $t$ under mild conditions. The proof relies on the fact that for any step $t$ and input $x$, model output $f^{(t)}(x)$ can be expressed as a polynomial function in $\eta$, similar to the previous section, although with coefficients that depend on initialization in a more complex way. By studying the behavior of this polynomial for $\eta$ small/large enough, we show that optimal $\eta$ converges almost surely to a non-zero deterministic constant under $\mu$P; hence proving LR transfer for general $t$.

\subsection{Understanding the difficulty at $t\geq 2$}
In the previous section, we showed that after one step the network output becomes asymptotically linear in $\eta$. This significantly simplified the asymptotic analysis of $\eta^{(1)}_n$ and allowed derivation of a closed-form expression for the limit $\eta^{(1)}_\infty$. For $t\geq 2$, such analysis is nontrivial since the linear asymptotics no longer hold. Indeed, for $t\geq 2$, higher-order monomials in $\eta$ are no longer negligible when $n$ is large. For instance, for $t=2$, we show that a coefficient of order $3L - 1$  in $f^{(2)}(x)$ converges to a non-zero constant as $n \to \infty$. Recall model output for a given input $x$
$$
f^{(2)}(x) = V^\top \left(\prod_{\ell = 1}^L W^{(2)}_\ell \right) W_0 x,
$$
where 
$$W^{(2)}_\ell = W^{(1)}_\ell - \eta m^{-1} \, \sum_{i=1}^m \chi_{i}^{(1)}\, b_{\ell+1}^{(1)} (a_{\ell - 1, i}^{(1)})^\top,$$
and, extending the notation from previous section, 
\begin{equation*}
\begin{cases}
b_\ell^{(t)} = (W_\ell^{(t)})^\top (W_{\ell+1}^{(t)})^\top \dots (W_L^{(t)})^\top \, V,\\
a_{\ell,i}^{(t)} = W_\ell^{(t)} W_{\ell-1}^{(t)} \dots W_1^{(t)} W_0 x_i,\\
\chi_{i}^{(t)} = f^{(t)}(x_i) - y_i.
\end{cases}
\end{equation*}

Unlike in the one-step analysis, model output at $t=2$ depends on the terms $b_\ell^{(1)}$, $a_\ell^{(1)}$, and $\chi^{(1)}$, which are all functions of the learning rate $\eta$. The leading monomial in $b_\ell^{(1)}$ is of degree $L-\ell + 1$ while in $a_\ell^{(1)}$ is of degree $\ell$. $\chi^{(1)}$ is a polynomial of degree $L$ in $\eta$. As a result, the leading monomial in $f^{(2)}(x)$ is of degree $L \times (1 + L + (L-\ell+1) + \ell) = 2L(L+1)$ in $\eta$. However, as in the analysis of the first step, the limiting polynomial as $n$ goes to infinity may not be of degree $2L(L+1)$. Expanding the product in $f^{(2)}(x)$ yields 
$$
f^{(2)}(x) = f^{(1)}(x) + \sum_{\ell=1}^L \phi_{\ell}(\eta) \eta^L,
$$
where $\phi_L(\eta) = (-1)^L V^\top \left(\prod_{\ell=1}^L  \gamma_\ell \right) W_0 x$, and 
$
\gamma_\ell = m^{-1} \, \sum_{i=1}^m \chi_{i}^{(1)}\, b_{\ell+1}^{(1)} (a_{\ell - 1, i}^{(1)})^\top
$.

Note that we emphasized the dependence of $\phi_L$ on learning rate $\eta$ in the notation. In the next result, we show that $\phi_L(\eta)$ converges to a non-zero constant as width goes to infinity, which is different from what we saw in the one-step loss.

\begin{lemma}[Non-linear asymptotics at $t=2$]\label{lemma:non_zero_coef_step2}
The limit of the coefficient $\phi_L(\eta)$ can be expressed as 
$$
\lim_{n\to \infty} \phi_{L}(\eta) = (-m)^L \sum_{i= 1}^m \gamma_{i} \frac{\langle x_{i}, x\rangle}{d},
$$
where, 
\begin{equation*}
\begin{cases}
    \gamma_i = \sum_{1\leq i_2 ,\dots, i_{L} \leq m} \zeta_{i, i_2, \dots, i_{L}},\\
    \zeta_{i_1, i_2, \dots, i_L} = \left(\prod_{j=1}^L \left(f_\infty^{(1)}(x_{i_j}) - y_{i_j}\right)\right) \left(\prod_{j=2}^L f_\infty^{(1)}(x_{i_j}) \right),
\end{cases}
\end{equation*}
with $f^{(1)}_\infty(x) = \eta \, \frac L m \sum_{i=1}^m y_i \frac{\langle x_i, x \rangle}{d}$.
\end{lemma}
\cref{lemma:non_zero_coef_step2} shows that $\phi_L(\eta)$ converges to a polynomial of degree $2L -1$ in $\eta$ as $n$ goes to infinity.\footnote{Note that here, we are implicitly assuming that $f^{(1)}_\infty(x_i) \neq y_i$ for all $i$, which is a realistic assumption since it is highly unlikely to interpolate the data after one gradient step.} Adding the $\eta^L$ term in $f^{(2)}(x)$, we obtain that $f^{(2)}(x)$ converges to a polynomial that has a non-zero term of order $3L -1$. Therefore, in contrast to step $1$, step 2 involves more complex dependencies in $\eta$, and a full characterization of the minimum is highly non-trivial in this case. This complexity should be expected to ``increase'' with step $t$ as gradient dependencies on $\eta$ become more complex with $t$. 

However, under an additional mild condition, we show that optimal LR converges to a non-zero constant for any step $t$, proving LR transfer for general $t$. Similar to the previous section, let $ K = \left(d^{-1} \langle x_i, x_j \rangle\right)_{1\leq i,j\leq m}$ be the input Gram matrix and $y = (y_1, y_2, \dots, y_m)^\top \in \reals^m$ be the vector containing all inputs from the training dataset.

\begin{figure*}
    \centering
    \includegraphics[width=0.3\linewidth]{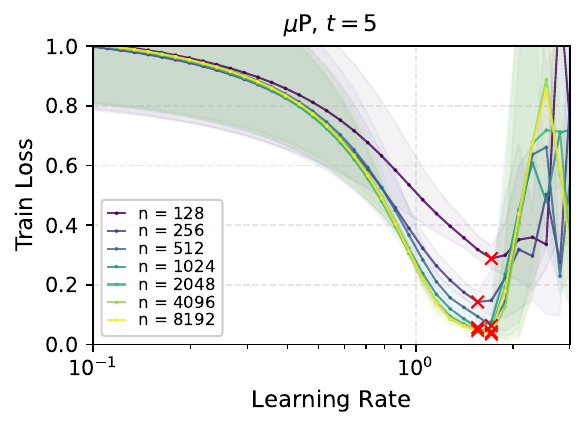}
    \includegraphics[width=0.3\linewidth]{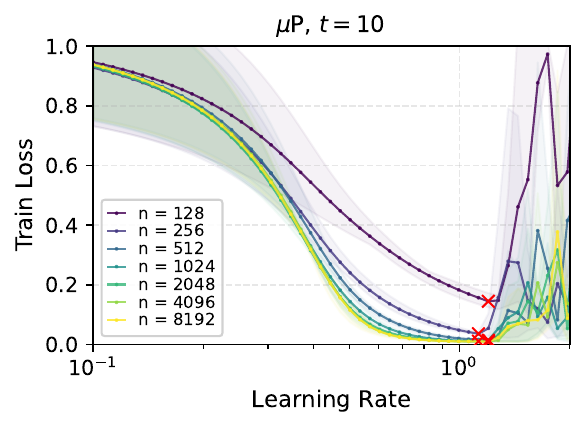}
    \includegraphics[width=0.3\linewidth]{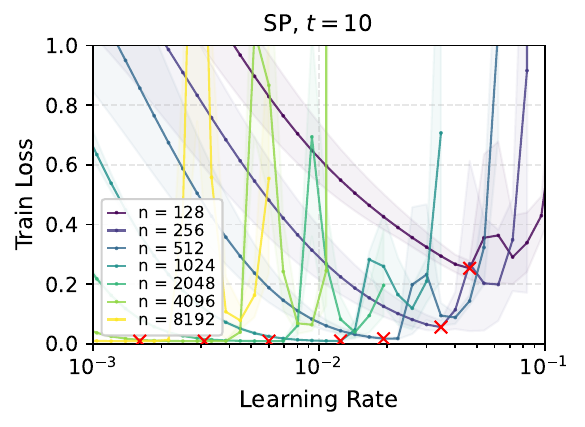}
    \caption{Train loss as function of LR at $t=5$ and $t=10$ for both $\mu$P and SP. Results are shown with 3 random seeds }
    \label{fig:linear_mlp_transfer_t}
\end{figure*}

\begin{thm}[LR transfer at step $t$]\label{thm:transfer_general_t} 
Assume that $ K y \neq 0$. Then the following holds:

\begin{enumerate}
    \item Given a fixed input $x$, the $t$-step model output $f^{(t)}(x)$ can be expressed as a polynomial function in $\eta$ where the coefficients depend only on initialization. As $n\to \infty$, all the coefficients converge almost surely to deterministic constants. We denote the limiting polynomial by $f_{\infty}^{(t)}$.
    \item The $t$-step loss $\loss^{(t)}_n(\eta)$ converges almost surely to $\loss^{(t)}_\infty(\eta) = \frac{1}{2m} \sum_{i=1}^m (f^{(t)}_{\infty}(\eta) - y_i)^2$ uniformly over $\eta$ on any compact set. Moreover, there exists \underbar{$\eta$}$, \bar{\eta}>0$ 
    such that $\operatorname{argmin}_{\eta \in [0,\infty)} \loss_\infty^{(t)} \subset [$\underbar{$\eta$}$, \bar{\eta}]$.
    
    \item Assume that $\loss^{(t)}_\infty$ has a unique minimizer $\eta^{(t)}_\infty$, let $I$ be an arbitrary compact set containing $\eta^{(t)}_\infty$, and let $\eta^{(t)}_n \in \operatorname{argmin}_{\eta \in I} \loss_n^{(t)}$.  Then, as $n\to \infty$, $$\eta^{(t)}_n \to \eta^{(t)}_\infty, \quad a.s.$$
\end{enumerate}

\end{thm}

The proof of \cref{thm:transfer_general_t} is provided in \cref{app:proofs_general_t}. The following sketch summarizes the proof machinery: the fact that $f^{(t)}(x)$ is a polynomial in $\eta$ is straightforward. The convergence of the coefficients to deterministic limit follows from the ``Master Theorem'' in \cite{yang2021tensor}. This convergence implies that $\loss_\infty^{(t)}$ is a polynomial with the leading monomial having a positive coefficient (quadratic loss). Therefore, the minimizer $\eta^{(t)}_\infty$ of $\loss_\infty^{(t)}$ is finite which yields a probabilistic bound on $\eta^{(t)}_n$ for $n$ large enough. We further show that the derivative of $\loss_{n}^{(t)}(\eta)$ at $\eta = 0$ converges to a negative real number which bounds the minimizer (in $\eta$) away from $0$. We conclude by observing that bounded roots of a converging sequence of polynomials converge to the roots of the limiting polynomial. Note that we show almost sure convergence, a much stronger convergence than convergence in probability or in $\mathbb L_2$ (almost sure convergence yields $\mathbb L_2$ convergence by Dominated Convergence Theorem). This stems from using almost sure convergence of scalar quantities from the Tensor Programs framework.

\cref{thm:transfer_general_t} shows that under the mild assumption that the limiting loss has a unique minimizer, LR transfer occurs under $\mu$P. This assumption is realistic as it is commonly observed in practice that training loss has a unique minimizer at any training step $t$.

\cref{fig:linear_mlp_transfer_t} shows the same results of \cref{fig:one_step} at different training steps. With $\mu$P, we observe that optimal LR $\eta^{(1)}_n$ converges as width $n$ grows for different training steps $t \in \{5, 10\}$, confirming the result of \cref{thm:convergence_step_t}. Note that we consider small number of steps here because training converges after 10 to 15 iterations since the dataset is relatively simple (linear) and we use full batch GD. With SP, we observe a similar pattern to the one-step analysis; the optimal LR vanishes with width, and therefore optimal LR doesn't transfer with width in this case.

In the next section, we provide additional experiments with more challenging setups, including non-linear synthetic data, networks with ReLU activation function, varying depth, and varying optimizers.

\section{Additional Experiments}\label{sec:exps}
We provide additional experiments to assess learning transfer with $\mu$P under several setups that are not necessarily covered by our theory. Our results shed light on the impact of the following factors: non-linearity (ReLU), network depth, training step, and optimizer.

\begin{figure*}
    \centering
    \includegraphics[width=0.3\linewidth]{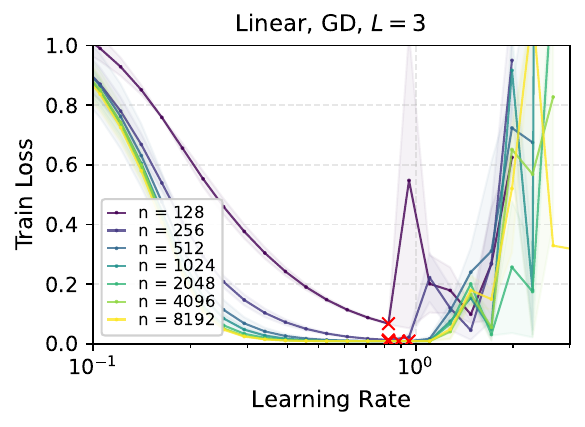}
    \includegraphics[width=0.3\linewidth]{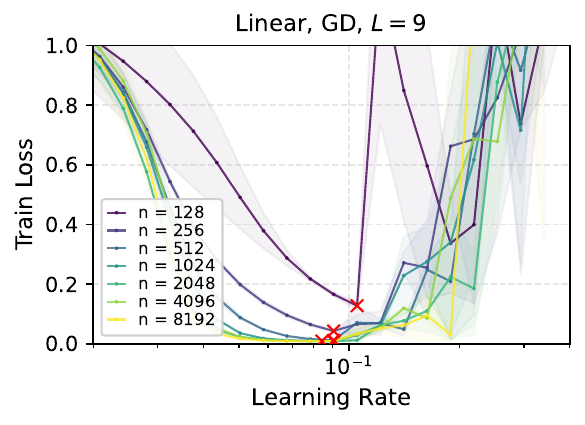}
    \includegraphics[width=0.3\linewidth]{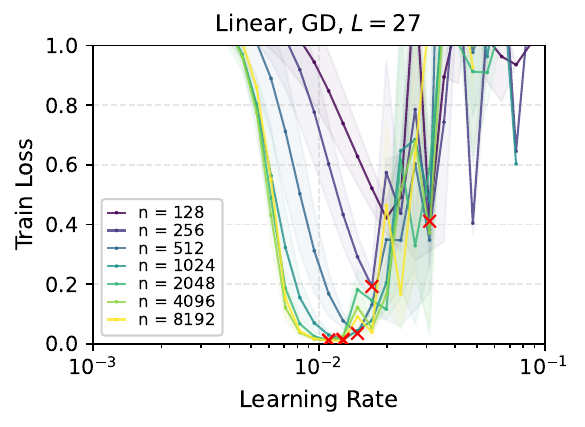}
    \includegraphics[width=0.3\linewidth]{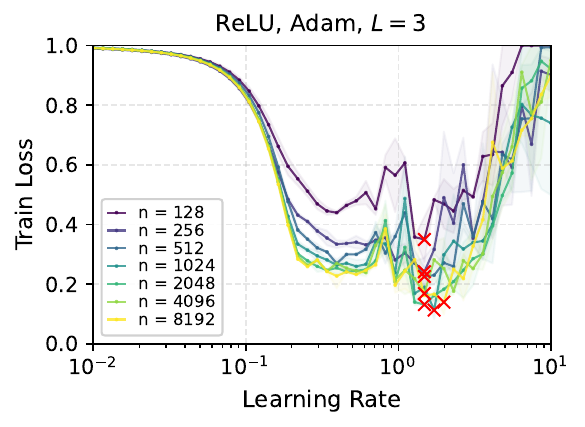}
    \includegraphics[width=0.3\linewidth]{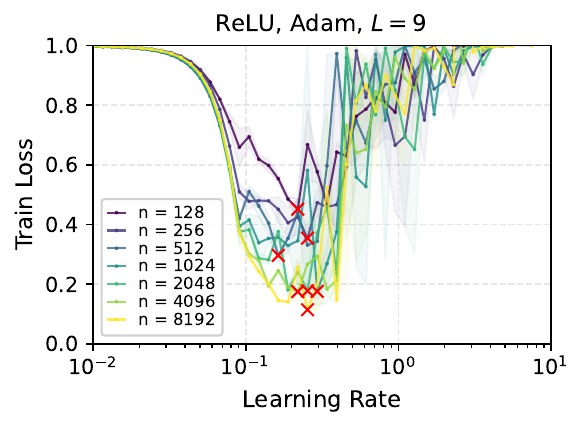}
    \includegraphics[width=0.3\linewidth]{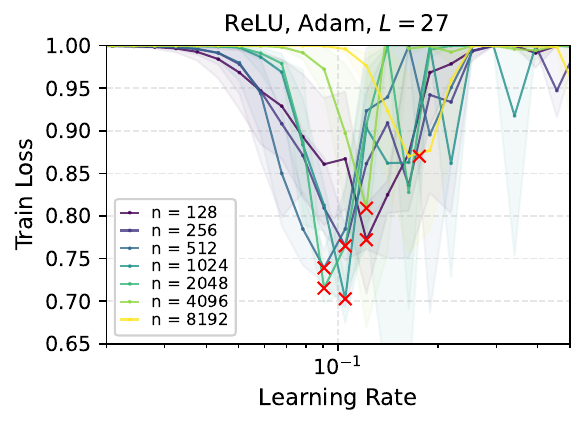}
    \caption{Train loss as a function of learning rate at $t=20$ with 3 random seeds. Red crosses highlight the optimal LR for each width. \textbf{(Top)} Linear MLP of varying depth trained with SGD. \textbf{(Bottom)} MLP with ReLU activation of varying depth trained with Adam.}
    \label{fig:sgd_adam_t20}
\end{figure*}

\paragraph{Training data.} We fix input dimension $d=100$ in all experiments. We generate a ground truth vector $\omega \sim \mathcal N(0, d^{-1} I_d)$ and generate $N$ inputs $x \sim \mathcal N(0, I_d)$ where $N=1000$ is fixed. We generate $N$ noise terms $\epsilon \sim \mathcal N(0,0.01)$ and consider two output generating processes:

\begin{itemize}
    \item \emph{Linear}: the outputs are generated as $y = \omega^\top x + \epsilon$. This setup is used for the linear networks (no activation function).
    \item \emph{Non-linear}: the outputs are generated as $y = \textrm{Sign}(\omega^\top x + \epsilon)$, where $Sign(.)$ is the sign function ($+1$ if non-negative and $-1$ otherwise). This setup is used for neural networks with ReLU activation function.
\end{itemize}

We train MLPs with varying depths $L \in \{3, 9, 27\}$ and discuss the results below.

\paragraph{Impact of Depth.} From \cref{fig:sgd_adam_t20}, we observe that LR transfer occurs at different depths, confirming the result of \cref{thm:convergence_step_t} which holds for any depth. Interestingly, the optimal LR seems to decrease with depth, which confirms depth-dependency predicted by the result of \cref{thm:m-argmin} (see expression of $\eta_\infty^{(1)}$).\footnote{There a depth version of $\mu$P called Depth-$\mu$P, see \citet{yang2023tensorprogramsvifeature}.}

\paragraph{ReLU and Adam.} \cref{fig:sgd_adam_t20} shows that LR transfer holds for non-linear MLPs (with ReLU) trained with Adam. While our theory does not cover this case, empirical results suggest that LR transfer remains valid for non-linear architectures and more advanced training algorithms.

\paragraph{Impact of Training Step.} \cref{fig:longer_training} shows LR transfer also holds near convergence. Interestingly, the range of close-to optimal learning rates widens with the number of steps, suggesting that when the number of training steps is large enough, optimal LR has low resolution in the sense that choosing the right order of magnitude for the LR should be enough to obtain near-best performance.

\begin{figure}
    \centering
    \includegraphics[width=0.8\linewidth]{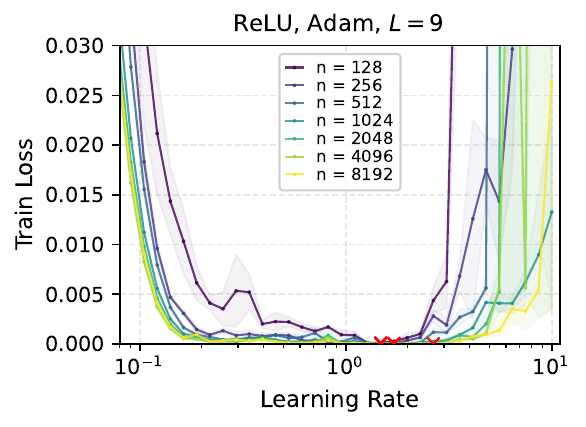}
    \caption{Train loss as a function of learning rate at $t=100$ with 3 random seeds. MLP of depth $L=9$ with ReLU activation trained with Adam.}
    \label{fig:longer_training}
\end{figure}

\section{Discussion and Limitations}
We presented the first of learning rate transfer under $\mu$P. Our theoretical results rely on expressing the training loss of a deep linear network as a polynomial function of the learning rate. By studying the infinite-width limit, we derived convergence results for the optimal LR. While our results are limited to linear networks trained with GD, we believe they can be extended to non-linear MLPs and different optimizers. However, this will likely require different proof machinery especially when dealing when large-width deviations. We leave this question for future work.

\section{Acknowledgment}
The author would like to thank Ruijia Zhang for their help with fixing some typos in this paper.

\bibliographystyle{plainnat}
\bibliography{sample}

\newpage

\onecolumn
\appendix
\section{Proofs}\label{app:proofs}

\subsection{Proof of \cref{lemma:m-vanish}}

We prove the result for $m=1$ (single sample dataset). Extending the result to general $m$ is straightforward.

\begin{lemma}\label{lemma:asymptotic_coefficients}
Assume $m=1$. Then, for all $\ell \in \{2, 3, \dots, L\}$, we have $\| \phi_\ell \|_{L_2} = \bigO(n^{-(\ell-1)/2})$.
\end{lemma}

\begin{proof}
Let $k \in \{2, \dots, L\}$. We show that all the terms inside $\phi_k$ are $\big(n^{-1/2})$ which concludes the proof. Let $1\leq \ell_1 < \ell_2 < \dots < \ell_k \leq L$. Then, we can write the summand as 

\begin{align*}
V^\top J_{L:\ell_k + 1} b_{\ell_{k}+1} &a_{\ell_k - 1}^\top  J_{\ell_k - 1:\ell_{k-1} + 1} \dots b_{\ell_1 + 1} a_{\ell_1 - 1}^\top J_{\ell_{1} - 1 : 1} W_0 x\\
&= \|b_{\ell_k +1}\|^2 \|a_{\ell_1 - 1}\|^2 \prod_{j=2}^k a_{\ell_j - 1}^\top J_{\ell_j - 1 : \ell_{j-1} + 1} b_{\ell_{j-1} + 1}.
\end{align*}

For some $j \in \{2, \dots, k\}$, let $J_j := J_{\ell_j - 1 : \ell_{j-1} + 1}$. We have 
$$
a_{\ell_j - 1}^\top J_{\ell_j - 1 : \ell_{j-1} + 1} b_{\ell_{j-1} + 1} = u^\top J_j^\top J_j J_j^\top v,  
$$
where $u = a_{\ell_{j-1}}$ and $v = b_{\ell_j}$.

Using Holder's inequality (product of $k$ random variables) and Lemma \ref{lemma:tri_concentration}, we obtain that
\begin{align*}
\E (V^\top J_{L:\ell_k + 1} b_{\ell_{k}+1} a_{\ell_k - 1}^\top  J_{\ell_k - 1:\ell_{k-1} + 1} \dots b_{\ell_1 + 1} a_{\ell_1 - 1}^\top J_{\ell_{1} - 1 : 1} W_0 x)^2 = \bigO(n^{-k+1}).
\end{align*}

We conclude by observing that $\lim_{n\to \infty} \chi = -y$.

\end{proof}

\paragraph{Proof for \cref{lemma:m-vanish}.} Identical to Lemma~\ref{lemma:asymptotic_coefficients}: each inner product block has second moment $\Theta(n^{-1})$ by Lemma~\ref{lemma:tri_concentration}. Products of $k-1$ such factors contribute $\Theta(n^{-(k-1)})$ to the second moment; the extra sum over $i_r\in[m]$ only changes constants, not the $n$-scaling. The convergence of $\phi_1$ is straightforward by Strong Law of Large Numbers (SLLN), and is a consequence of \cref{lemma:m-gram} below, which proves convergence of a kernel matrix to the Gram matrix $K$ of input data.

\subsection{Proof of \cref{thm:m-argmin}}

The proof proceeds as follows: we first characterize the infinite-width limit of $\phi_1$, then we study the asymptotics of the loss function and conclude on the convergence of the optimal learning rate.

\paragraph{First-order term and a layerwise Gram matrix.}
Fox $(x_j, y_j)$ in the training dataset, the degree one coefficient $\phi_1$ in the expression of $f^{(1)}(x_j)$ as a polynomial in $\eta$ is given by 
\begin{equation}\label{eq:phi1-m}
\phi_{1} \;=\; -\frac1m\sum_{\ell=1}^L \sum_{i=1}^m
\chi_i\, \|b_{\ell+1}\|^2\, \big\langle a_{\ell-1, i},\, a_{\ell-1, j}\big\rangle .
\end{equation}
Let $G_{\ell-1}\in\mathbb R^{m\times m}$ be the layerwise Gram with $(G_{\ell-1})_{ij}=\langle a_{\ell-1, i}, a_{\ell-1, j}\rangle$, and define the normalized input Gram
\(
K\in\mathbb R^{m\times m},\quad K_{ij}=\langle x_i,x_j\rangle/d.
\)
The next results characterizes the infinite-width limit of a kernel matrix from which the limit of $\phi_1$ follows.
\begin{lemma}[Layerwise Gram limit; $m$ points]\label{lemma:m-gram}
As $n\to\infty$,
\[
\frac{1}{L}\sum_{\ell=1}^L \|b_{\ell+1}\|^2\, G_{\ell-1}
\ \xrightarrow{\text{a.s.}}\ K.
\]
\end{lemma}

\begin{proof}
For $\ell \in \{1,\dots, L\}$, we have $\E\|b_{\ell+1}\|^2=1/n$. The vectors $a_{\ell-1, i}$ are jointly Gaussian with per-coordinate covariance $\langle x_i,x_j\rangle/d$. Independence between $b_{\ell+1}$ and $(a_{\ell-1, i})_{i=1}^m$ gives $\E[\|b_{\ell+1}\|^2 G_{\ell-1}]=K$. A simple application of the SLLN implies the a.s.\ convergence of the layerwise average to $K$.
\end{proof}

\paragraph{Limiting one-step loss and optimal step size.}
Let $\chi=(\chi_1^{(1)},\dots,\chi_m^{(1)})^\top$, $y=(y_1,\dots,y_m)^\top$. Using Lemma~\ref{lemma:m-vanish} and \eqref{eq:phi1-m}, uniformly for $\eta$ on compact intervals,
\begin{equation}\label{eq:loss-m}
\loss_n(\eta)
= \frac{1}{2m}\big\|\chi - \eta\,H_n\,\chi\big\|^2 + o_{\mathbb L_2}(1),
\qquad
H_n \ =\ \sum_{\ell=1}^L \frac1m\,\|b_{\ell+1}\|^2\, G_{\ell-1}.
\end{equation}
By Lemma~\ref{lemma:m-gram}, $H_n\xrightarrow{\text{a.s.}} \tfrac{L}{m}K$, and since $\chi\to -y$ in $\mathbb L_2$ (as $f^{(0)}(x_i)\to 0$ in $\mathbb L_2$), we obtain the deterministic limit
\begin{equation}\label{eq:limit-loss-m}
\loss_\infty^{(1)}(\eta)
\overset{def}{=} \lim_{n\to\infty}\loss_n^{(1)}(\eta)
= \frac{1}{2m}\big\| -y + \eta\,\tfrac{L}{m}\,K\,y\big\|^2. \quad \textrm{a.s.}
\end{equation}

The next result shows convergence of the optimal learning rate $\eta^{(1)}_n$.
\begin{lemma}[LR transfer; limiting minimizer]\label{lemma:m-argmin}
Assume $Ky\neq 0$, then $\loss_\infty^{(1)}(\eta)$ is strictly convex quadratic with the unique minimizer
\begin{equation}\label{eq:eta-star-m}
\eta_\infty^{(1)}
\;=\; \frac{m}{L}\,
\frac{y^\top K y}{\| K y\|^2}.
\end{equation}
Moreover, for any compact set $I \subset [0, \infty)$ containing $\eta_\infty^{(1)}$, we have for any $\eta_n^{(1)}\in\operatorname{argmin}_{\eta \in I} \loss_n^{(1)}(\eta)$, $\eta_n^{(1)}\to \eta_\infty^{(1)}$ in $\mathbb L_2$.
\end{lemma}

\begin{proof}
The limiting loss \eqref{eq:limit-loss-m} is a strictly convex quadratic in $\eta$ whenever $K y \neq 0$. Differentiating yields \eqref{eq:eta-star-m}. Uniform convergence in $\mathbb L_2$ of $\loss_n^{(1)}\to\loss_\infty^{(1)}$ on compacts (in $\eta$) plus strict convexity implies convergence of minimizers.
\end{proof}

\paragraph{Particular case.} When the inputs are orthogonal, i.e. if $\langle x_i,x_j\rangle=0$ for $i\neq j$, then $K=\operatorname{diag}(k_1,\dots,k_m)$ with $k_i=\|x_i\|^2/d$, and
\[
\eta_\infty^{(1)}=\frac{m}{L}\cdot
\frac{\sum_{i=1}^m y_i^2 k_i}{\sum_{i=1}^m y_i^2 k_i^2}.
\]

\subsection{Convergence rate}

As above, we assume $Ky\neq 0$ and work with the one–step loss
\[
\loss_n^{(1)}(\eta)\;=\;\frac{1}{2m}\sum_{j=1}^m\big(f^{(1)}(x_j)-y_j\big)^2
\]
We also recall the limiting quadratic
\(
\loss_\infty^{(1)}(\eta)=\frac{1}{2m}\big\|-y+\eta\,\tfrac{L}{m}Ky\big\|^2
\)
with unique minimizer
\(
\eta_\infty^{(1)}=\frac{m}{L}\frac{y^\top Ky}{ \|K y\|^2}.
\)

Let $\chi_\infty=(-y_1,\dots,-y_m)^\top$  and recall
\[
H_n=\sum_{\ell=1}^L \frac1m\,\|b_{\ell+1}\|^2\,G_{\ell-1}\in\mathbb R^{m\times m},
\qquad (G_{\ell-1})_{ij}=\big\langle a_{\ell-1, i},a_{\ell-1, j}\big\rangle .
\]

Let us explicitly state the bounds (instead of $o(1)$ in the previous section) as these are needed to characterize the convergence rate.
\begin{lemma}[One-step decomposition with uniform remainders]\label{lem:decomp}
Fix any compact interval $I\subset(0,\infty)$. Then, uniformly in $\eta\in I$,
\begin{equation}\label{eq:decomp}
\loss_n(\eta)
=\frac{1}{2m}\big\|\chi-\eta H_n\chi\big\|^2 + R_n(\eta),
\end{equation}
where the remainder satisfies
\[
\sup_{\eta\in I}\big|R_n(\eta)\big|=O_{\mathbb L_2}(n^{-1/2}),
\qquad
\sup_{\eta\in I}\big|R_n'(\eta)\big|=O_{\mathbb L_2}(n^{-1/2}).
\]
\end{lemma}

\begin{proof}
The results follows \cref{lemma:m-vanish}. The term $ R_n$ collects all terms containing coefficients of monomial $\eta^k$ with $k\ge 2$.
By Lemma~\ref{lemma:m-vanish}, for each $k\ge 2$ and $j$, $\|\phi_{k}\|_{L_2}=O(n^{-(k-1)/2})$; thus for fixed $L$ and $\eta\in I$, $R_n(\eta)$ and $ R_n'(\eta)$ are dominated by the $k=2$ contribution and are $O_{\mathbb L_2}(n^{-1/2})$ uniformly on $I$.
\end{proof}

The next result characterizes the convergence rate of the effective kernel $H_n$ to the infinite-width kernel $K$.
\begin{lemma}[Convergence rates for $\chi$ and $H_n$]\label{lem:baseline}
As $n\to\infty$,
\[
\max_{1\le i\le m}\big|f^{(0)}(x_i)\big|^2=O_{\mathbb L_2}(n^{-1}),\qquad H_n = \tfrac{L}{m}K + O_{\mathbb L_2}(n^{-1/2}),
\]
where the last equality holds element-wise.
\end{lemma}

\begin{proof}
\emph{First claim.} For each $i$, conditionally on $a_{L,i}$, $f^{(0)}(x_i)=V^\top a_{L,i}$ is Gaussian with mean $0$ and variance $\tfrac{1}{n^2}\|a_{L,i}\|^2$ since $V\sim\mathcal N(0,n^{-2}I_n)$ is independent of $a_{L,i}$. Taking expectations and using isotropy of the $W_\ell$ (so $\mathbb E\|a_{L,i}\|^2=\|x_i\|^2$), we obtain $\mathbb E[f^{(0)}(x_i)^2]=\|x_i\|^2/n^2$, hence $|f^{(0)}(x_i)|^2=O_{\mathbb L_2}(n^{-1})$. Since $m$ is fixed, we can take the max over $i$. 

\emph{Second claim.} For $T_\ell\overset{def}{=} m^{-1}\|b_{\ell+1}\|^2 G_{\ell-1}$, independence of the ``top'' block ($b_{\ell+1}$) and the ``bottom'' block ($G_{\ell-1}$) implies $\mathbb E[T_\ell]=(1/m)K$ (as in Lemma~\ref{lemma:m-gram}). For any fixed $(i,j)$,
\[
(T_\ell)_{ij}=\frac{1}{m}\|b_{\ell+1}\|^2 \langle a_{\ell-1, i},a_{\ell-1, j}\rangle .
\]
Conditionally on the weights $W_{\ell -2} ... W_0$, $\langle a_{\ell-1, i},a_{\ell-1, j}\rangle$ is a sum of iid random variables with mean $n^{-1}\langle a_{\ell-2, i},a_{\ell-2, j}\rangle$. Therefore, 

$$\E \left[(n^{-1} \langle a_{\ell-1, i},a_{\ell-1, j}\rangle - n^{-1} \langle a_{\ell-2, i},a_{\ell-2, j}\rangle)^2 \mid W_{\ell -2} ... W_0 \right] = \bigO(n^{-1}).$$ 
Doing this recursively yields 
$$\E \left[(n^{-1} \langle a_{\ell-1, i},a_{\ell-1, j}\rangle - K_{ij})^2 \right] = \bigO(n^{-1}),$$ 
which concludes the proof.

\end{proof}

\begin{lemma}[Uniform convergence and strong convexity]\label{lem:uniform+sc}
Fix compact $I\subset[0,\infty)$. Then
\[
\sup_{\eta\in I}\big|\loss_n(\eta)-\loss_\infty(\eta)\big|=O_{\mathbb L_2}(n^{-1/2}),
\quad
\sup_{\eta\in I}\big|\partial_\eta\loss_n(\eta)-\partial_\eta\loss_\infty(\eta)\big|=O_{\mathbb L_2}(n^{-1/2}),
\]
and
\[
\inf_{\eta\in I}\partial_{\eta\eta}^2\loss_n(\eta)\ \xrightarrow{\mathbb L_2}\ 
\mu\ :=\ \frac{L^2}{m^3}\,y^\top K^2 y\ >0 .
\]
\end{lemma}

\begin{proof}
Using \eqref{eq:decomp} and expanding the quadratic part,
\[
\loss_n(\eta)-\loss_\infty(\eta)
=\frac{1}{2m}\Big(\|\chi\|^2-\|y\|^2 - 2\eta\big[\chi^\top H_n \chi-y^\top\tfrac{L}{m}Ky\big]
+ \eta^2\big[\chi^\top H_n^2\chi - y^\top\tfrac{L^2}{m^2}K^2 y\big]\Big)+R_n(\eta).
\]
By Lemma~\ref{lem:baseline}, $\E \max_i|f^0(x_i)|^2=\bigO(n^{-1})$, hence $\chi=-y+O_{\mathbb L_2}(n^{-1/2})$. Also $H_n=(L/m)K+O_{\mathbb L_2}(n^{-1/2})$ coordinate wise (and thus in operator norm). Therefore each bracketed term above is $O_{\mathbb L_2}(n^{-1/2})$ uniformly on $I$, and $R_n(\eta)=O_{\mathbb L_2}(n^{-1/2})$ by Lemma~\ref{lem:decomp}, which proves the first result. Differentiating the decomposition gives the derivative bound by the same argument. Finally,
\[
\partial_{\eta\eta}^2\loss_n(\eta)=\frac{1}{m}\,\chi^\top H_n^2 \chi + R_n''(\eta),
\]
and the right-hand side converges in $\mathbb L_2$ to $(1/m)\,y^\top ((L/m)K)^2 y$, uniformly on $I$.
\end{proof}

\begin{lemma}[Rates for the argmin and for the loss at the argmin]\label{lemma:rates-full}
Let $I\subset(0,\infty)$ be any compact interval containing $\eta_\infty^{(1)}$. Let $\eta_n^{(1)}\in\arg\min_{\eta\in I}\loss_n(\eta)$. Then, as $n\to\infty$,
\[
\eta_n^{(1)}-\eta_\infty^{(1)} = O_{\mathbb P}(n^{-1/2}),
\qquad
\loss_n(\eta_n^{(1)})-\loss_\infty(\eta_\infty^{(1)}) = O_{\mathbb P}(n^{-1/2}),
\]
and
\[
\loss_\infty(\eta_n^{(1)})-\loss_\infty(\eta_\infty^{(1)})=\frac{\mu}{2}\,(\eta_n^{(1)}-\eta_\infty^{(1)})^2  = \bigO_{\mathbb P}(n^{-1}).
\]
Consequently, the loss gap at the argmin is dominated by the \emph{uniform} $n^{-1/2}$ error of $\loss_n$ (the shift of the minimizer contributes only $O_{\mathbb P}(n^{-1})$).
\end{lemma}

\begin{proof}
By Lemma~\ref{lem:uniform+sc}, there exists (with high probability) a constant $c>0$ such that $\inf_{\eta\in I}\loss_n''(\eta)\ge c$ for all large $n$. Using the mean-value form of the optimality condition,
\[
0=\loss_n'(\eta_n^{(1)})=\loss_n'(\eta_\infty^{(1)})+\loss_n''(\tilde\eta_n)\,(\eta_n^{(1)}-\eta_\infty^{(1)})
\]
for some $\tilde\eta_n$ between $\eta_\infty^{(1)}$ and $\eta_n^{(1)}$. Hence
\[
|\eta_n^{(1)}-\eta_\infty^{(1)}|
\;\le\; \frac{1}{c}\,|\loss_n'(\eta_\infty^{(1)})|
\;\le\; \frac{1}{c}\Big(\sup_{\eta\in I}\big|\loss_n'(\eta)-\loss_\infty'(\eta)\big|\Big).
\]
Using the fact that $\sup_{\eta\in I}|\loss_n'(\eta)-\loss_\infty'(\eta)|=O_{\mathbb L_2}(n^{-1/2})$ by Lemma~\ref{lem:uniform+sc} yields $\eta_n^{(1)}-\eta_\infty^{(1)}=O_{\mathbb P}(n^{-1/2})$.

For the loss at the argmin, write
\[
\loss_n(\eta_n^{(1)})-\loss_\infty(\eta_\infty^{(1)})
=\underbrace{\big(\loss_n(\eta_\infty^{(1)})-\loss_\infty(\eta_\infty^{(1)})\big)}_{O_{\mathbb P}(n^{-1/2})}
+ \underbrace{\big(\loss_\infty(\eta_n^{(1)})-\loss_\infty(\eta_\infty^{(1)})\big)}_{\text{shift term}} .
\]
The first term is $O_{\mathbb P}(n^{-1/2})$ by Lemma~\ref{lem:uniform+sc}. For the shift term, a Taylor expansion of $\loss_\infty$ around $\eta_\infty^{(1)}$ gives
\[
\loss_\infty(\eta_n^{(1)})-\loss_\infty(\eta_\infty^{(1)})=\tfrac{1}{2}\loss_\infty''(\eta_\infty^{(1)})\,(\eta_n^{(1)}-\eta_\infty^{(1)})^2
=\frac{\mu}{2}\,(\eta_n^{(1)}-\eta_\infty^{(1)})^2,
\]
and since $\eta_n^{(1)}-\eta_\infty^{(1)}=O_{\mathbb P}(n^{-1/2})$, this is $O_{\mathbb P}(n^{-1})$. So the dominant term is the $\bigO_{\mathbb P}(n^{-1/2})$ above, which concludes the proof. 
\end{proof}

\subsection{Failure of LR Transfer under Standard Parametrizations}
We consider Standard Parametrization where the different with $\mu$P lies only in how the head $V$ is initialized: $V\sim\mathcal N(0,n^{-1})$, while
$W_0\sim\mathcal N(0,d^{-1})$ and $W_\ell\sim\mathcal N(0,n^{-1})$ for $\ell=1,\dots,L$. For the learning rate, we assume $c=0$, i.e. the learning rate is parametrized as a constant $\eta>0$.

We provide the proof for $m=1$. Extending the result to $m\geq 1$ is straightforward. Let $(x,y)$ be the training datapoint. At $t=1$, the output is given by 

$$f^{(1)}(x) = V^\top \left[\prod_{\ell = 1}^L \left(W^{(0)}_\ell - \eta\,  \chi \, b_{\ell+1} a_{\ell - 1}^\top\right)\right] W_0 x,$$

where $\chi = f^{(0)}(x) - y$, which can be written as $f^{(1)}(x) = f^{(0)}(x) + \sum_{\ell = 1}^L \phi_l \eta^\ell$.

With SP, it is straightforward to see that all coefficients $\phi_{\ell}$ are of order $\sqrt{n}$ in $L_2$. It suffices to normalize $V$ by $\sqrt{n}$ and we're essentially back to the case of $\mu P$ with the same asymptotic analysis (\cref{lemma:tri_concentration}).

Expressing the loss function as $\loss^{(1)}_n(\eta) = (f^{(1)}(x) - y)^2 = (a_0 + a_1 \eta + \dots + a_L \eta^L)^2$, it is easy to check that this polynomial satisfies the conditions in \cref{lemma:sp_convergence}, which yields the result.

\newpage
\section{Proofs for \cref{sec:general_t}}\label{app:proofs_general_t}

\textbf{Lemma \ref{lemma:non_zero_coef_step2}. }[Non-linear behavior after step $t=2$]
\emph{The limit of the coefficient $\phi_L(\eta)$ can be expressed as 
$$
\lim_{n\to \infty} \phi_{L}(\eta) = (-m)^L \sum_{1\leq i_1, i_2 ,\dots, i_L \leq m} \zeta(i_1, i_2, \dots, i_L) \frac{\langle x_{i_1}, x\rangle}{d},
$$
where 
$$
\zeta(i_1, i_2, \dots, i_L) = \left(\prod_{j=1}^L \left(f_\infty^{(1)}(x_{i_j}) - y_{i_j}\right)\right) \left(\prod_{j=2}^L f_\infty^{(1)}(x_{i_j}) \right),
$$
with $f^{(1)}_\infty(x) = \eta \, \frac L m \sum_{i=1}^m y_i \frac{\langle x_i, x \rangle}{d}$.}

The proof of \cref{lemma:non_zero_coef_step2} is straightforward by taking the infinite-width limit.

From \cref{lemma:non_zero_coef_step2}, we obtain that $\phi_L(\eta)$ converges to a polynomial of degree $2L -1$ in $\eta$ as $n$ goes to infinity. Adding the $\eta^L$ term in $f^{(2)}$, we obtain that $f^{(2)}$ converges to a polynomial that has a non-zero term of degree $3L -1$. Therefore, in contrast to step $1$, step 2 involves more complex dependencies in $\eta$, and a full analysis of the minimum is non-trivial in this case. This complexity should be expected to increase with step $t$ as gradient dependencies on $\eta$ become more complex with $t$.

The next result shows convergence of $f^{(t)}(x)$ to a limiting polynomial $P^{(t)}$, with deterministic coefficients. This is a straightforward result from the convergence of constants in a Tensor Program.
\begin{thm}\label{thm:convergence_step_t}
Let $t \geq 1$ and $x \in \reals^d$. Then, for any $K>0$, there exists a polynomial $f^{(t)}_\infty$ with deterministic coefficients such that
$$
\lim_{n \to \infty} \sup_{\eta \in [0, K]}|f^{(t)}(x) - f^{(t)}_\infty(\eta)| = 0. \quad a.s.
$$
\end{thm}

\begin{proof}
Let $t \geq 1$ and $x \in \reals^d$. $f^{(t)}(x)$ is a polynomial in $\eta$ with coefficients that can be expressed via the Tensor Program framework. The convergence follows from Theorem 7.4 in \cite{yang2021tensor}.
\end{proof}
Note that the convergence can also be made uniform in input $x$ living in compact sets. This is not useful here since we consider a finite training dataset.

We now state the formal LR transfer result and prove it.

\begin{thm}[HP Transfer for general $t$] 
Let $K = \left(\frac{\langle x_i, x_j \rangle}{d}\right)_{1\leq i,j\leq m}$ and $y = (y_1, y_2, \dots, y_m)^\top \in \reals^m$, and assume that $K y \neq 0$. 
Let $f^{(t)}_\infty$ be the limiting polynomial (in $\eta$) of $f^{(t)}(x)$ from the result above. Then, $\loss^{(t)}_n(\eta)$ converges almost surely to $\loss^{(t)}_\infty(\eta) = \frac{1}{2m} \sum_{i=1}^m (f^{(t)}_{\infty}(\eta) - y_i)^2$ uniformly over $\eta$ in some arbitrary compact set. Moreover, there exists \underbar{$\eta$} $, \bar{\eta}>0$ such that $\operatorname{argmin}_{\eta \in [0,\infty)} f_\infty^{(t)} \subset [$\underbar{$\eta$}$, \bar{\eta}]$.

Moreover, assume that $\loss^{(t)}_\infty$ has a unique minimizer $\eta^{(t)}_\infty$, let $\gamma \gg \eta^{(t)}_\infty$ be an arbitrarily large constant, and let $\eta^{(t)}_n \in \operatorname{argmin}_{\eta \in [0,\gamma]} \loss_n^{(t)}$.  We have that

$$\lim_{n \to \infty} \eta^{(t)}_n = \eta^{(t)}_\infty, \quad a.s.$$
\end{thm}

\begin{proof}
From \cref{thm:convergence_step_t}, we know that $f^{(t)}(x)$ converges almost surely to $f^{(t)}_\infty$ on any compact set. The convergence of $\loss^{(t)}$ follows. 

Now looking at the limiting loss $\loss^{(t)}_\infty$ as a polynomial in $\eta$, the leading monomial has positive coefficient because of the squared loss. Therefore $\lim_{\eta \to \infty} \loss_\infty^{(t)}(\infty) = \infty$ which implies that there exists $\bar{\eta}>0$ such that $\operatorname{argmin}_{\eta \in [0, \infty)]} \loss_\infty^{(t)} \subset [0,\bar{\eta}]$.

Now, let us prove the existence of \underbar{$\eta$}. Observe that $\loss_\infty^{(t)}(0) = \frac{1}{2m}\sum_{i=1}^m y_i^2 > 0$.  Moreover, from \cref{lemma:derivative_at_zero}, we have that

$$
\frac{\partial \loss_\infty^{(t)}}{\partial \eta}\Big|_{\eta=0} = \frac{1}{m} \sum_{i=1}^m \frac{t\, L}{m} \sum_{j=1}^m y_j \frac{\langle x_j, x_i \rangle}{d} (-y_i) = -\frac{t\, L}{m^2} \, y^\top K y.
$$
Under the assumption that $K y \neq 0$, we have $\frac{\partial \loss_\infty^{(t)}}{\partial \eta}\Big|_{\eta=0} <0$. As a result, by continuity of $\loss_\infty^{(t)}$ with respect to $\eta$, there exists a neighborhood of $\eta=0$ that does not contain the minimizer of $\loss_\infty^{(t)}$. In other words, there exists \underbar{$\eta$}$ > 0$ such that $(\operatorname{argmin}_{\eta \in [0,\infty)}\loss_{\infty}^{(t)}) \cap [0,$\underbar{$\eta$}$) = \emptyset$.\\

Finally, under the assumption that $\loss_\infty^{(t)}$ has a unique minimizer in $(0,\infty)$, the convergence result follows from \cref{thm:stability_argmin}.

\end{proof}

The next lemma characterizes the derivative of the infinite-width polynomial limit $f^{(t)}_\infty$ at $\eta=0$. It is used in the proof of LR transfer for general t.
\begin{lemma}[Derivative of $f^{(t)}$ at $\eta=0$]\label{lemma:derivative_at_zero}
Let $x \in \reals^d$ and $t\geq 1$. We have the following
$$
\frac{\partial f^{(t)}_\infty}{\partial \eta}\Big|_{\eta=0} = \lim_{n\to\infty}\frac{\partial f^{(t)}}{\partial \eta}\Big|_{\eta=0} = \frac{t\, L}{m} \sum_{i=1}^m y_i\, \frac{\langle x_i, x\rangle}{d}, \quad a.s.
$$
\end{lemma}
\begin{proof}
We can express the output as 
$$f^{(t)}(x) = V^\top \left[\prod_{\ell = 1}^L \left(W^{(0)}_\ell - \eta \sum_{s=0}^{t-1}\,m^{-1} \, \sum_{i=1}^m \chi_{i}^{(s)}\, b_{\ell+1}^{(s)} (a_{\ell - 1, i}^{(s)})^\top\right)\right] W_0 x.$$ 

Expanding in $\eta$, we have
$$\chi_i^{(s)} = f^{(s)}(x_i) - y_i = f^{(0)}(x_i) -  y_i + \eta \times \tilde \chi_{i}^{(s)}(\eta),$$ 
for some polynomial $\chi_{i}^{(s)}$. Similarly, 
$$b^{(s)}_{\ell} = b^{0}_{\ell} + \eta \tilde b^{(s)}_{\ell}(\eta),$$
and 
$$
a^{(s)}_{\ell} = a^{0}_{\ell} + \eta \tilde a^{(s)}_{\ell}(\eta).
$$

Therefore, we can express $f^{(t)}$ as follows
$$f^{(t)}(x) = V^\top \left[\prod_{\ell = 1}^L \left(W^{(0)}_\ell - \eta\, t\,m^{-1} \, \sum_{i=1}^m \chi_{i}^{(0)}\, b_{\ell+1}^{(0)} (a_{\ell - 1, i}^{(0)})^\top + \eta^2 \Psi_\ell(\eta) \right)\right] W_0 x,$$

where $\Psi_\ell$ is a polynomial in $\eta$. It follows that
$$
\frac{\partial f^{(t)}}{\partial \eta}\Big|_{\eta=0} = -\frac{t}{m} \, \sum_{\ell = 1}^L V^\top J^{(0)}_{\ell+1} \sum_{i=1}^m \chi_i^{(0)} b_{\ell+1}^{(0)} (a_{\ell,i}^{(0)})^\top W_0 x.
$$
Taking the width $n$ to infinity yields the desired result, with almost sure convergence.
\end{proof}

The next result is used in the proof of LR transfer for general step $t$. It shows the almost sure convergence of the argmin of a polynomial under some conditions.

\begin{thm}[Argmin stability with a.s.\ coefficient convergence and positive polynomials]\label{thm:stability_argmin}
Fix an integer $p\ge1$. For each $n\ge1$, let
\[
P_n(x)=\sum_{k=0}^{p} a_{n,k}\,x^{k},\qquad x\in[0,\infty),
\]
where the coefficients $a_{n,k}$ are real-valued random variables on a common probability space.
Assume there exist deterministic reals $(a_k)_{k=0}^{p}$ such that, for every $k=0,\dots,p$,
\[
a_{n,k}\xrightarrow[n\to\infty]{\text{a.s.}} a_k,
\]
and set the (deterministic) limit polynomial
\[
P_\infty(x)=\sum_{k=0}^{p} a_k x^{k}.
\]
Suppose:
\begin{enumerate}
\item[(1)] For each $n$, $P_n(x)\ge 0$ for all $x\ge0$ almost surely.
\item[(2)] $P_\infty$ has a unique minimizer $x_\star\in[0,\infty)$.
\end{enumerate}
Then, for any constant $R>0$, and for any $x_n\in\operatorname{argmin}_{[0,R]}P_n$ we have
\[
x_n \xrightarrow{\ \text{a.s.}\ } x_\star .
\]
\end{thm}

\begin{proof}
Let $\Omega_0$ be the probability-one event on which $a_{n,k}\to a_k$ for all $k$ and $P_n(x)\ge0$ for all $x\ge0$ and all $n$.
Let's fix $\omega\in\Omega_0$ and argue deterministically.

\emph{(i) Uniform convergence on compacts:}
For any $R>0$, we have
\[
\sup_{x\in[0,R]}|P_n(x)-P_\infty(x)|
\le \sum_{k=0}^{p} |a_{n,k}-a_k|\,R^{k}\xrightarrow[n\to\infty]{}0,
\]
so $P_n\to P$ uniformly on every compact subset of $[0,\infty)$.

\emph{(ii) Convergence of minimizers.}
Let $R>0$. By uniqueness, for each $\delta>0$ the compact set $K_\delta=\{x\in[0,R]:|x-x_\star|\ge\delta\}$ satisfies
\[
\Delta_\delta\overset{def}{=}\min_{x\in K_\delta}\bigl(P(x)-P(x_\star)\bigr)>0.
\]
Uniform convergence on $[0,R]$ yields $n_\delta$ with $\sup_{x\in[0,R]}|P_n(x)-P(x)|\le \Delta_\delta/3$ for all $n\ge n_\delta$.
Thus, for $n\ge\max\{N,n_\delta\}$ and $x\in K_\delta$,
\[
P_n(x)\ge P(x)-\tfrac{\Delta_\delta}{3}\ge P(x_\star)+\tfrac{2\Delta_\delta}{3}\ge P_n(x_\star)+\tfrac{\Delta_\delta}{3},
\]
so no minimizer lies in $K_\delta$, i.e.\ $|x_n-x_\star|<\delta$. As $\delta>0$ is arbitrary, $x_n\to x_\star$.
Since $\omega\in\Omega_0$ was arbitrary, the convergence holds almost surely.
\end{proof}

\newpage

\section{Technical Lemmas}

The following lemma is used in the proofs of 1-step convergence results.

\begin{lemma}\label{lemma:tri_concentration}
Let $L\ge 1$ be fixed. For $\ell=1,\dots,L$, let $W^{(\ell)}\in\reals^{n\times n}$ have i.i.d.\ entries with
$\E W^{(\ell)}_{ij}=0$ and $\E (W^{(\ell)}_{ij})^2 = n^{-1}$, and assume the entries are uniformly sub-gaussian.
Assume the matrices $\{W^{(\ell)}\}_{\ell=1}^L$ are independent.
Let
\[
J := W^{(L)} W^{(L-1)} \cdots W^{(1)}\in\reals^{n\times n}.
\]
Let $x,y\in\reals^n$ be independent of $\{W^{(\ell)}\}_{\ell=1}^L$, with i.i.d.\ coordinates of zero mean, unit variance,
and uniformly sub-gaussian. Set
\[
S:=x^\top J^\top J J^\top y,\qquad A:=J^\top J J^\top.
\]
Then for every fixed $p>0$ there exists a constant $C_{p,L}<\infty$ such that, for all sufficiently large $n$,
\[
\E|S|^p \;\le\; C_{p,L}\,n^{p/2},
\qquad\text{equivalently}\qquad
\|S\|_{\mathbb L^p}\;\le\; C_{p,L}\sqrt n .
\]
\end{lemma}

\newcommand{\op}{\mathrm{op}}

\begin{proof}
Constants may depend on $p$, $L$, and the sub-gaussian parameters, but not on $n$.
Let $S=x^\top Ay$ with $A=J^\top J J^\top$.

\paragraph{Step 1: $\mathbb L^p$ bound for $x^\top Ay$.}
Everything in this step is conditioned on $J$ (so $A$ is deterministic).
Fix $p\ge 1$ and condition on $y$. Writing $v:=Ay$, we have $S=\sum_{i=1}^n v_i x_i$.
By the subgaussian Khintchine inequality \cite[Prop.~2.7.5]{VershyninHDP2} and
$\psi_2 \Rightarrow \mathbb L^p$ \cite[Prop.~2.6.6(ii)]{VershyninHDP2},
\begin{equation}\label{eq:linform_lp_simple_fixed}
\|S\|_{\mathbb L^p\,|\,y}\le C\sqrt p\,\|Ay\|_2.
\end{equation}
Taking $\mathbb L^p$ in $y$ gives
\begin{equation}\label{eq:reduce_simple_fixed}
\|S\|_{\mathbb L^p}\le C\sqrt p\,\big\|\|Ay\|_2\big\|_{\mathbb L^p}.
\end{equation}
Next, by anisotropic concentration of the norm \cite[Ex.~6.13]{VershyninHDP2} (applied with $B=A$ and $X=y$),
\[
\big\|\|Ay\|_2-\|A\|_F\big\|_{\psi_2}\le C\|A\|_{\op}.
\]
Using again \cite[Prop.~2.6.6(ii)]{VershyninHDP2}, we obtain
\begin{equation}\label{eq:Ay_lp_fixed}
\big\|\|Ay\|_2\big\|_{\mathbb L^p}\le \|A\|_F + C\sqrt p\,\|A\|_{\op}.
\end{equation}
Combining \eqref{eq:reduce_simple_fixed} and \eqref{eq:Ay_lp_fixed} yields
\begin{equation}\label{eq:bilinear_form_fixed}
\|S\|_{\mathbb L^p}\le C\Big(\sqrt p\,\|A\|_F + p\,\|A\|_{\op}\Big).
\end{equation}
For $p\in(0,1)$, use $\|Z\|_{\mathbb L^p}\le \|Z\|_{\mathbb L^1}$.

\paragraph{Step 2: $\mathbb L^p$ bound for $\|A\|_F$.}
We have $\|A\|_F^2=\tr((J^\top J)^3)$. Since $\tr(B)\le n\|B\|_{\op}$ for $B\succeq 0$,
\[
\|A\|_F^2=\tr\big((J^\top J)^3\big)\le n\,\|(J^\top J)^3\|_{\op}
= n\,\|J^\top J\|_{\op}^3
= n\,\|J\|_{\op}^6,
\]
hence
\begin{equation}\label{eq:AF_bound_by_J}
\|A\|_F \le \sqrt n\,\|J\|_{\op}^3.
\end{equation}
By the subgaussian operator-norm bound for i.i.d.\ matrices (applied to $\sqrt n\,W^{(\ell)}$) and tail-to-moment integration
(e.g.\ \cite[Thm.~4.4.3 and Lem.~1.6.1]{VershyninHDP2}), for every $q\ge 1$ we have $\sup_n \E\|W^{(\ell)}\|_{\op}^q<\infty$.
Using submultiplicativity and independence,
\[
\|J\|_{\op}\le \prod_{\ell=1}^L \|W^{(\ell)}\|_{\op}
\qquad\Rightarrow\qquad
\sup_n \E\|J\|_{\op}^q <\infty \quad\text{for every fixed }q\ge 1.
\]
Therefore, taking $q=3p$ in \eqref{eq:AF_bound_by_J} gives
\[
\|\|A\|_F\|_{\mathbb L^p}\le \sqrt n\,\|\|J\|_{\op}^3\|_{\mathbb L^p}
= \sqrt n\,\|J\|_{\mathbb L^{3p}}^3
\le C_{p,L}\sqrt n.
\]
\paragraph{Step 3: $\mathbb L^p$ upper bound for $S$.}
From \eqref{eq:bilinear_form_fixed} and the trivial bound $\|A\|_{\op}\le \|A\|_F$, we get
\[
\|S\|_{\mathbb L^p}\le C(\sqrt p + p)\,\big\|\|A\|_F\big\|_{\mathbb L^p}.
\]
By Step~2 and uniform integrability of polynomial spectral statistics, for each fixed $p>0$,
\[
\big\|\|A\|_F\big\|_{\mathbb L^p}\le C_{p,L}'\sqrt n.
\]
Therefore $\|S\|_{\mathbb L^p}\le C_{p,L}''\sqrt n$.
\end{proof}

The next lemma is used in the proof of the 1-step result for SP.
\begin{lemma}[Lemma for SP]\label{lemma:sp_convergence}
Let $P(\eta) = a_0 + a_1 \eta + a_2 \eta^2+ \dots + a_L \eta^L$ be a polynomial where the coefficients $a_0, a_1, \dots, a_L$ are random variables satisfying the following conditions:
\begin{enumerate}
    \item $E[a_0^2] = O(1)$ and $a_0$ converges weakly to some random variable $\bar a_0$ of order 1 in distribution as $n \to \infty$.
    \item $E[a_i^2] = O(n)$ for $i = 1, \dots, L$, and $a_1/\sqrt{n}$ converges in $\mathbb L_2$ to a deterministic constant $\bar{b}_1 \neq 0$ as $n \to \infty$, with $a_1/\sqrt{n} = \bar b_1 + \bigO_{\mathbb L_2}(n^{-1/2})$.
\end{enumerate}
Let $K>0$ be a constant and $\eta_n$ be a minimizer of $P(\eta)^2$ on $[0,K]$, i.e., $\eta_n \in \arg\min_{\eta \in [0,K]} P(\eta)^2$. Then, $\eta_n$ converges to 0 in probability as $n \to \infty$.
\end{lemma}

\begin{proof}
The proof proceeds by rescaling the domain of the polynomial to analyze its behavior in a neighborhood of $0$, similar to the treatment of the $\mu$P case. 

Consider the change of variables $\eta = \beta/\sqrt{n}$. Let $\eta_n$ be a minimizer of $P(\eta)^2$. The corresponding minimizer in the $\beta$ domain is $\beta_n = \eta_n \sqrt{n}$.

We now prove that the sequence of random variables $\{\hat{\beta}_n\}$ is bounded in probability, i.e.  $\beta_n = O_p(1)$. This will imply the convergence of $\eta_n$.

Let's define a new sequence of random polynomials in the variable $\beta$ by substituting $\eta = \beta/\sqrt{n}$ into $P(\eta)$
\[
R_n(\beta) = P(\beta/\sqrt{n}) = a_0 + a_1 \frac{\beta}{\sqrt{n}} + a_2 \frac{\beta^2}{(\sqrt{n})^2} + \dots + a_L \frac{\beta^L}{(\sqrt{n})^L}
\]

Define a new set of coefficients $b_i^{(n)} = a_i/\sqrt{n}$ for $i \geq 1$. We can now rewrite the rescaled polynomial as
\[
R_n(\beta) = a_0 + b_1^{(n)} \beta + b_2^{(n)} \frac{\beta^2}{\sqrt{n}} + b_3^{(n)} \frac{\beta^3}{n} + \dots + b_L^{(n)} \frac{\beta^L}{n^{(L-1)/2}}
\]
For any fixed $\beta \in \mathbb{R}$, as $n \to \infty$, every term for $i \geq 2$ converges to zero in $\mathbb L_2$. For instance, for the term $i=2$, we have $b_2^{(n)} \beta^2 / \sqrt{n} \xrightarrow{L_2} 0$ because $b_2^{(n)}$ is bounded in $\mathbb L_2$. This holds for all $\ell \in \{2, \dots, L\}$.

Therefore, the sequence of random polynomials $R_n(\beta)$ in asymptotically controlled as follows
\[
R_n(\beta) - R(\beta) = O_{\mathbb L_2}(n^{-1/2}),
\]
where $R(\beta) = a_0 + b_1 \beta$.

Let $\beta^*_n \in \textrm{argmin}_{\eta \in [0,K]} R_n(\beta)^2$ for $K$ large enough (so that the global minimizer is covered). The second derivative of $R_n(.)^2$ is given by $2 R_n'' R_n + 2 (R_n')^2$. We know that uniformly on $[0,K]$, $R_n''(\beta) = o_{\mathbb L_2}(1)$, and $R_n'(\beta) =  b^{(n)}_1 + \bigO_{\mathbb L_2}(n^{-1/2})$. Therefore, uniformly over $\beta \in [0,K]$, we have that $(R_n(\beta)^2)'' = 2 (b^{(n)}_1)^2 + \bigO_{\mathbb L_2}(n^{-1/2})$ = 2 $\bar b_1^2 + \bigO_{\mathbb L_2}(n^{-1/2})$.

As a result, as $n \to \infty$, with high probability, there exists a constant $c>0$ such that $\inf_{[0,K]}(R_n(\beta)^2)'' \geq c$. Using the Intermediate Value Theorem, we have that 
$$
|\beta^*_n|=|\beta^*_n - 0| \leq \frac{|(R_n^2)'(0)|}{c} = \frac{|b_1^{(n)} a_0|}{c}.
$$

Which shows that $\beta^*_n = \bigO_{\Prob}(1)$ and concludes the proof.
\end{proof}

\end{document}